\documentclass[final]{cvpr}

\usepackage{times}
\usepackage{epsfig}
\usepackage{graphicx}
\usepackage{amsmath}
\usepackage{amssymb,amsthm}

\usepackage{amsmath,amsfonts,bm,amsthm}
\newtheorem{theorem}{Theorem}

\newtheorem{assumption}[theorem]{Assumption}
\newtheorem{proposition}[theorem]{Proposition}
\newtheorem{lemma}[theorem]{Lemma}

\def\eqref#1{equation~\ref{#1}}

\def\1{\bm{1}}

\DeclareMathAlphabet{\mathsfit}{\encodingdefault}{\sfdefault}{m}{sl}
\SetMathAlphabet{\mathsfit}{bold}{\encodingdefault}{\sfdefault}{bx}{n}

\DeclareMathOperator*{\argmax}{arg\,max}
\DeclareMathOperator*{\argmin}{arg\,min}

\usepackage{algorithm}
\usepackage{algorithmicx}
\usepackage{algpseudocode}
\usepackage{url}
\usepackage{xcolor}
\usepackage{multirow}
\usepackage{caption}
\usepackage{subcaption}
\usepackage{booktabs}
\usepackage{adjustbox}
\usepackage{calc}

\usepackage[pagebackref=true,breaklinks=true,colorlinks,bookmarks=false]{hyperref}

\def\cvprPaperID{xxxx} 
\def\confYear{CVPR 2021}

\begin{document}

\title{Soft-IntroVAE: Analyzing and Improving the Introspective Variational Autoencoder}

\author{Tal Daniel\\
Department of Electrical Engineering\\
Technion, Haifa, Israel\\
{\tt\small taldanielm@campus.techion.ac.il}
\and
Aviv Tamar\\
Department of Electrical Engineering\\
Technion, Haifa, Israel\\
{\tt\small avivt@technion.ac.il}
}

\maketitle

\begin{abstract}
The recently introduced introspective variational autoencoder (IntroVAE) exhibits outstanding image generations, and allows for amortized inference using an image encoder. The main idea in IntroVAE is to train a VAE adversarially, using the VAE encoder to discriminate between generated and real data samples. However, the original IntroVAE loss function relied on a particular hinge-loss formulation that is very hard to stabilize in practice, and its theoretical convergence analysis ignored important terms in the loss.
In this work, we take a step towards better understanding of the IntroVAE model, its practical implementation, and its applications. We propose the Soft-IntroVAE, a modified IntroVAE that replaces the hinge-loss terms with a smooth exponential loss on generated samples. This change significantly improves training stability, and also enables theoretical analysis of the complete algorithm. Interestingly, we show that the IntroVAE converges to a distribution that minimizes a sum of KL distance from the data distribution and an entropy term. We discuss the implications of this result, and demonstrate that it induces competitive image generation and reconstruction. Finally, we describe two applications of Soft-IntroVAE to unsupervised image translation and out-of-distribution detection, and demonstrate compelling results. Code and additional information is available on the project website - \url{https://taldatech.github.io/soft-intro-vae-web}.
\end{abstract}


\begin{figure*}
     \centering
     \begin{subfigure}[b]{0.52\textwidth}
         \centering
         \includegraphics[width=\textwidth]{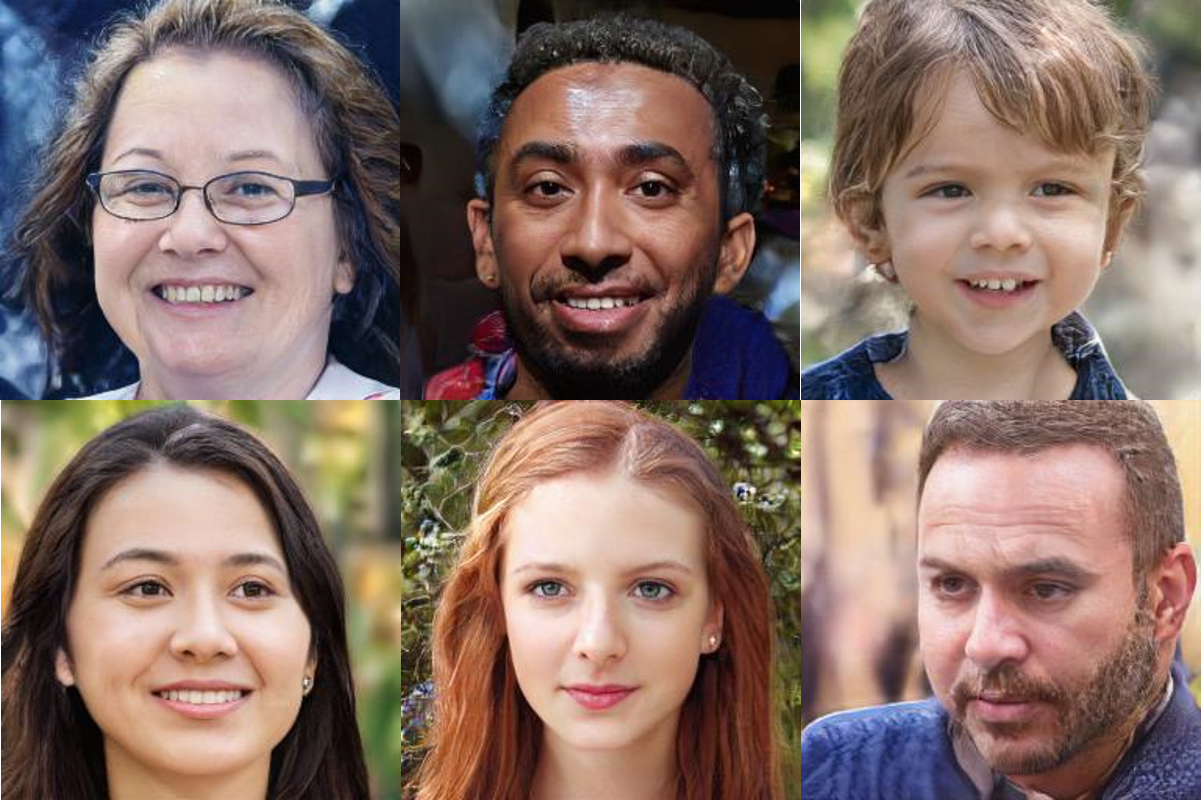}
         \caption{FFHQ dataset -- samples from S-IntroVAE (FID: 17.55).}
         \label{fig:samples_ffhq}
     \end{subfigure}
     \begin{subfigure}[b]{0.4\textwidth}
         \centering
         \includegraphics[width=0.6\textwidth]{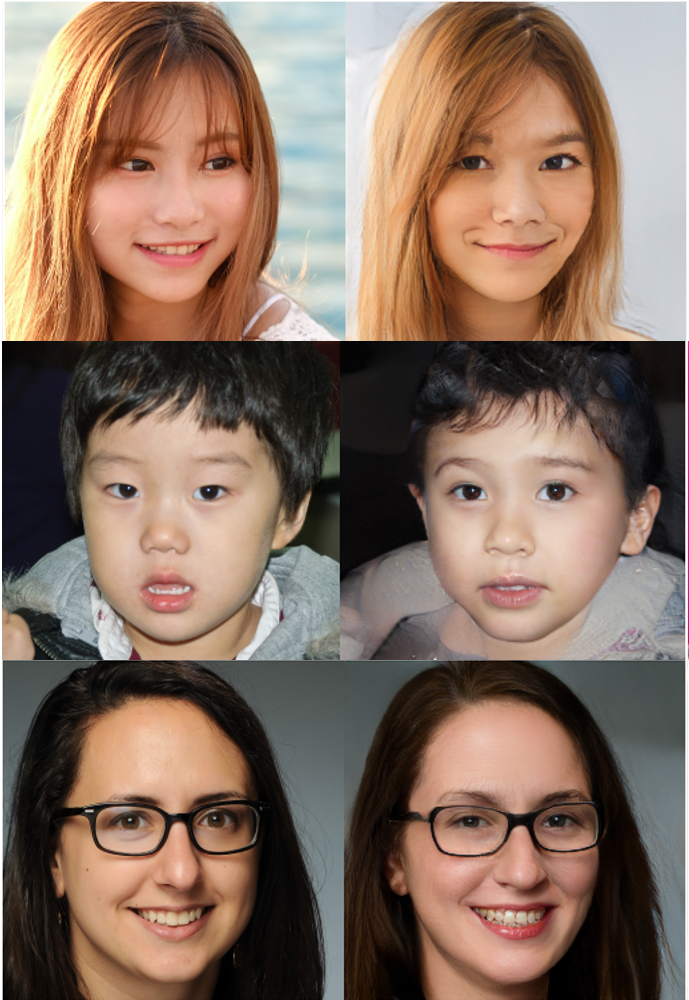}
         \caption{FFHQ -- reconstructions.}
         \label{fig:ffhq_recons}
     \end{subfigure}
     \vspace{-0.5em}
        \caption{Generated samples (left) and reconstructions (right) of test data (left: real, right: reconstruction) from a style-based S-IntroVAE trained on FFHQ at 256x256 resolution. }
        \label{fig:image_samples_ffhq}
        \vspace{-1.5em}
\end{figure*}

\section{Introduction}

Two popular approaches for learning deep generative models are generative adversarial training (e.g., GANs~\cite{goodfellow2014generative}), and variational inference (e.g., VAEs~\cite{kingma2014autoencoding}). VAEs are known to have a stable training procedure, display resilience to mode collapse, and enable amortized inference. Moreover, the VAE's inference module makes them prominent in many domains, such as learning disentangled representations \cite{bouchacourt2017multi} and reinforcement learning \cite{ha2018world}. 
GANs, on the other hand, lack an inference module, but are capable of generating images of higher quality and are popular in computer vision applications, but can suffer from  training instability and low sampling diversity (mode collapse)~\cite{mescheder2018training, kodali1705convergence}.

Narrowing the gap between VAEs and GANs has been the aim of many works, in an attempt to combine the best of both worlds: building a stable and easy to train generative model that allows efficient amortized inference and high-quality sampling~\cite{larsen2015autoencoding, dumoulin2016adversarially,makhzani2015adversarial, zhao2016energybased, pidhorskyi2020adversarial}. While progress has been made, the search for better generative models is an active research field.

Recently, Huang et al.~\cite{introvae18} proposed IntroVAE -- a VAE that is trained adversarially, and demonstrated outstanding image generation results. 
A key idea in IntroVAE is introspective discrimination -- instead of training a separate discriminator network to discriminate between real and generated samples, as in a GAN, the output of the encoder acts as the discriminatory signal, based on the Kullback-Leibler divergence between the approximate posterior of a sample and the prior latent distribution. Intuitively, this discriminatory signal can be understood as making the generated samples less likely, as their posterior is more distant from the prior.

Importantly, IntroVAE uses a \textit{hard margin}, $m$, as a threshold on the KL divergence for which above it a `fake' sample no longer affects the loss function. The hard margin approach leads to difficult training in practice -- the optimization process is very sensitive to the values of $m$, and extensive tuning is required to find a good, stable, margin. On the theoretical side, \cite{introvae18} proved that IntroVAE converges to the data distribution, but their analysis ignored several terms in the loss function that are important for correct operation of the algorithm.

We aim to provide a better understanding of the introspective training paradigm and improve its stability. To that end, we introduce Soft-IntroVAE (S-IntroVAE) -- an introspective VAE that utilizes the evidence lower bound (ELBO) as the discriminatory signal, and replaces the hard margin with a soft threshold function, making it more stable to optimize. This new formulation allows us to analyze the convergence properties of the complete algorithm, and provide new insights into introspective VAEs.

Interestingly, our theoretical analysis shows that, in contrast to the original IntroVAE, the S-IntroVAE encoder converges to the true posterior, thus maintaining the inference capabilities of the VAE framework. Our analysis also reveals that the S-IntroVAE model converges to a generative distribution that minimizes a sum of KL divergence from the data distribution and an entropy term, in contrast to GANs, where the distribution of the generator converges to the data distribution. We further analyze the consequences of this result and its practical implications, and show that the model produces high-quality generations from a distribution with a sharp support. 
In practice, S-IntroVAE is much more stable than IntroVAE, as it does not involve the sensitive threshold parameter, and we rigorously validate this claim in experiments, ranging from inference on 2D distributions to high-quality image generation.

We further demonstrate a practical application of our model to the task of unsupervised image translation. We exploit the fact that S-IntroVAE has both good generation quality and strong inference capabilities, and combine it with an encoder-decoder architecture that induces disentanglement. Inductive bias of this sort is required for unsupervised learning of disentangled representations \cite{JMLR:v21:19-976}, and our results show that using this architecture, S-IntroVAE is indeed capable of successfully transferring content between two images without any supervision.

Finally, the adversarial training of S-IntroVAE effectively assigns a lower likelihood to out-of-distribution (OOD) samples, allowing to use the model for OOD detection. While recent work claimed that likelihood-based models, such as VAEs, are incapable of distinguishing between images from different datasets \cite{nalisnick2018deep}, we show that a well-trained S-IntroVAE model exhibits near-perfect detection on all the tasks investigated in \cite{nalisnick2018deep}, significantly outperforming the standard VAE.

Our contribution is summed as follows: (1) We propose Soft-IntroVAE, a modification of the original IntroVAE that that utilizes the evidence lower bound (ELBO) as a discriminatory signal, and does not require a hard-margin threshold; (2) We provide a deeper theoretical understanding of introspective VAEs; (3) We validate that training our model is significantly more robust than training the original IntroVAE; (4) We show that our method is capable of high-quality image synthesis; (5) We demonstrate practical applications of our model to the tasks of unsupervised image translation and OOD detection.

\section{Related Work}
Studies on enhancing VAEs can be divided to methods that either improve the network's architecture~\cite{vahdat2020nvae, sonderby2016ladder}, incorporate stronger priors~\cite{tomczak2018vae, razavi2019generating, kalatzis2020variational}, add regularization~\cite{ghosh2019variational, xu2020learning}, or incorporate adversarial objectives~\cite{introvae18, makhzani2015adversarial, Han_2020, pidhorskyi2020adversarial, Heljakka_2020, dumoulin2016adversarially, larsen2015autoencoding}. Our work belongs to the latter group.
Adversarial Autoencoders~\cite{makhzani2015adversarial} add an adversarial loss on the latent code using a discriminator in addition to the typical encoder/decoder training.  VAE/GAN chains VAEs and GANs by adding a discriminator on top of the decoder~\cite{larsen2015autoencoding}. BiGAN~\cite{donahue2016adversarial} and ALI~\cite{dumoulin2016adversarially} simulate autoencoding by training with adversarial objectives in both the latent and data spaces. VEEGAN~\cite{srivastava2017veegan} introduces an additional 'reconstructor' network that is trained to map the real data to a Gaussian and to approximately invert the generator network.
Recently, \cite{pidhorskyi2020adversarial} proposed ALAE, a latent adversarial autoencoding method that adds a style-based encoder to the generator of StyleGAN~\cite{karras2020analyzing}. 
AGE~\cite{ulyanov2017takes}, IAN~\cite{brock2016neural}, PIONEER~\cite{Heljakka_2019}, B-PIONEER~\cite{Heljakka_2020}, and IntroVAE~\cite{introvae18} propose an introspective training of VAEs, and the latter is the current state of the art in image generation.
While the idea of introspective training is shared between the aforementioned methods, the difference is summed as follows: in AGE, the adversarial game between the encoder and generator is performed in the latent space (minimizing and maximizing divergences). In addition, the encoder minimizes the reconstruction error in \textit{data-space} (e.g. pixel-space), while the generator minimizes the reconstruction error in the \textit{latent-space}. PIONEER's objective is similar to AGE's but uses the cosine distance applied to reconstructions in the latent space. In IntroVAE, as in the previous methods, the encoder also maximizes divergence of the generated data latent representation, but uses a hard-margin over it, and both encoder and decoder minimize the reconstruction error in data-space, as in VAE. Unlike the previous methods, the introspective training in IAN uses a discriminator that shares weights with the encoder and the objective function is similar to VAE/GAN.
In this work, we focus on IntroVAE, as it is state-of-the-art in image synthesis.
We provide a deeper theoretical analysis of the introspective training and propose a stable algorithm that is capable of high-quality image generation.

Table \ref{tab:work_comp}, originally devised by \cite{pidhorskyi2020adversarial}, compares between different autoencoding methods based on the following properties: 
(a) how does the model learn the distribution of the data (i.e., by similarity to the original data or in an adversarial manner)?; (b) does the model impose a distribution type on the latent space (e.g., a Gaussian) or is the latent space prior learned? and (c) is the autoencoding reconstruction error measured (termed reciprocity) with respect to the data or the latent space?  
Soft-IntroVAE is a modification of IntroVAE, and shares similar high-level characteristics, as shown in the table.

\begin{table}[t!]
  \centering
  \resizebox{0.45\textwidth}{!}{
    \footnotesize
  \begin{tabular}{ |c||c|c|c|  }
 \hline
Autoencoding method& (a) Data & (b) Latent & (c) Reciprocity \\
 & Distribution & Distribution & Space\\
 \hline
 VAE~\cite{kingma2014autoencoding}   & similarity  & imposed/divergence & data \\
 AAE~\cite{makhzani2015adversarial} &  similarity  & imposed/adversarial  & data \\
 VAE/GAN~\cite{larsen2015autoencoding} & similarity & imposed/divergence & data \\
 VampPrior~\cite{tomczak2018vae} & similarity & learned/divergence & data \\
 BiGAN~\cite{donahue2016adversarial}  & adversarial & imposed/adversarial & adversarial \\
 ALI~\cite{dumoulin2016adversarially} &  adversarial & imposed/adversarial & adversarial \\
 VEEGAN~\cite{srivastava2017veegan} & adversarial & imposed/divergence & latent \\
 AGE~\cite{ulyanov2017takes} & adversarial  & imposed/adversarial & latent \\
  PIONEER~\cite{Heljakka_2019} & adversarial  & imposed/adversarial & latent \\
 IntroVAE~\cite{introvae18} & adversarial & imposed/adversarial & data \\
 ALAE~\cite{pidhorskyi2020adversarial} & adversarial & learned/divergence & latent \\
 \textbf{Soft-IntroVAE (ours)} & adversarial & imposed/adversarial & data\\
 \hline
\end{tabular}}
\caption{Comparison between AE methods based on criteria used: (a) for matching the real to the synthetic data distribution; (b) for setting/learning the latent distribution; (c) for which space reciprocity is achieved.}
\label{tab:work_comp}
\vspace{-3mm}
\end{table}

\section{Background}\label{sec:background}
We formulate our problem under the variational inference setting \cite{kingma2014autoencoding}, given some data $x\sim p_{data}(x)$, where $p_{data}$ denotes the data distribution, one aims to fit the parameters $\theta$ of a latent variable model $p_\theta(x) = \mathbb{E}_{p(z)} \left[p_\theta (x|z)\right]$, where the prior $p(z)$ is given, and $p_\theta (x|z)$ is learned. For general models, a typical objective is the maximum-likelihood $\max_\theta \log p_\theta(x)$, which is intractable, and can be approximated using variational inference methods. 

\paragraph{Evidence lower bound (ELBO)}
For some model defined by $p_\theta (x|z)$ and $p(z)$, let $p_\theta(z|x)$ denote the posterior distribution that the model induces on the latent variable. The evidence lower bound (ELBO) states that for any approximate posterior distribution $q(z|x)$: 
\begin{equation}\label{eq:ELBO_def}
\begin{split}
    \log p_\theta(x) &\geq \mathbb{E}_{q(z|x)}\ \left[\log p_\theta(x|z)\right] - KL(q(z|x) \Vert p(z)) \\ &\doteq ELBO(x), \raisetag{1.5em}
\end{split}
\end{equation}
where the Kullback-Leibler (KL) divergence is
$KL(q(z|x) \Vert p(z)) = \mathbb{E}_{q(z|x)} \left[ \log \frac{q(z|x)}{p(z)} \right]$. In the variational autoencoder (VAE,~\cite{kingma2014autoencoding}), the approximate posterior is represented as $q_\phi(z|x) = \mathcal{N}(\mu_\phi(x), \Sigma_\phi(x))$ for some neural network with parameters $\phi$, the prior is $p(z) = \mathcal{N}(\mu_0,\Sigma_0)$, and the ELBO can be maximized using the \textit{reparameterization trick}. Since the resulting model resembles an autoencoder, the approximate posterior $q_\phi(z|x)$ is also known as the \emph{encoder}, while $p_\theta(x|z)$ is termed the \emph{decoder}. Typically, $p_\theta(x|z)$ is modelled as a Gaussian distribution, and in this case $\log p_\theta(x|z)$ is equivalent to the mean-squared error (MSE) between $x$ and the mean of $p_\theta(x|z)$. In the following, the term reconstruction error refers to $\log p_\theta(x|z)$.

\paragraph{Introspective VAE (IntroVAE)}
The Introspective VAE adds an adversarial objective to the VAE training, where the intuition, following \cite{introvae18}, is that since the VAE maximizes a lower bound of the data likelihood, there is no guarantee that points outside the support of $p_{data}$ will not be assigned high likelihood. In~\cite{introvae18}, the KL term in the ELBO is seen as an `energy' of a sample. Inspired by energy-based GANs (EBGAN, ~\cite{zhao2016energybased}), the encoder is encouraged to classify between the generated and real samples by \textit{minimizing} the KL of latent distribution of real samples and the prior, and \textit{maximizing} the KL of generated ones. The decoder, on the other hand, is trained to reconstruct real data samples using the standard ELBO, and to minimize the KL of generated samples that go through the encoder. Consider a real sample $x$ and a generated one $D_{\theta}(z)\sim p_\theta(x|z)$,  
and let $KL(E_\phi(D_\theta(z)))=KL(q_{\phi}(\cdot|D_\theta(z)) \Vert p(\cdot))$ denote the KL of the generated sample that goes through the encoder. The adversarial objective, which is to be \textit{maximized}, for $x$ and $z$ is given by:
\begin{equation}
    \label{eq:introvae_practice}
    \begin{split}
        \mathcal{L}_{E_{\phi}}(x,z) & = ELBO(x) - [m - KL(E_{\phi}(D_{\theta}(z)))]^{+}, \\
        \mathcal{L}_{D_{\theta}}(x,z) & = ELBO(x) -  KL(E_{\phi}(D_{\theta}(z))),
    \end{split}
\end{equation}
where $\mathcal{L}_{E_{\phi}}$ is the objective function for the encoder, $\mathcal{L}_{D_{\theta}}$ is the objective function for the decoder, and $[\cdot]^{+} = \max(0, \cdot)$. The hard threshold $m$ in $\mathcal{L}_{E_{\phi}}$ limits the KL of generated samples in the objective, and is crucial for the analysis and practical implementation of IntroVAE. The encoder and decoder are trained simultaneously using stochastic gradient descent (SGD) on mini-batches of $x$ and $z$.

The theoretical properties of IntroVAE were studied in~\cite{introvae18} under a simplified objective that ignores the reconstruction terms in the ELBO, and omits the decoder loss for real samples, as given by:
\begin{equation}
    \label{eq:introvae_theory}
    \begin{split}
        \mathcal{L}_{E_{\phi}}(x,z) & = -KL(E_{\phi}(x)) - [m - KL(E_{\phi}(D_{\theta}(z)))]^{+}, \\
        \mathcal{L}_{D_{\theta}}(z) & = -KL(E_{\phi}(D_{\theta}(z))).
    \end{split}
\end{equation}
For this objective, \cite{introvae18} showed that a Nash equilibrium is obtained when the distribution $p(x)=\mathbb{E}_z [p_\theta(x|z)]$ is exactly $p_{data}(x)$, showing the soundness of the model in terms of sample generation. However, the EBGAN-based analysis did not make claims about the inference capabilities of the model, and in particular, it is possible that at the Nash equilibrium the encoder distribution $q_{\phi}(z|x)$ is different from the true posterior $p_{\theta}(z|x)$.\footnote{The theoretical results in \cite{introvae18}, as the results in our work, are proved under the non-parametric setting. We nevertheless use the parametric setting notation to avoid introducing additional notation, and understand from the context that $\theta$ and $\phi$ do not carry meaning in the non-parametric setting.}

\section{Soft-IntroVAE}

In this section, we propose a new introspective VAE model that mitigates two shortcomings of IntroVAE -- the training instability due to the hard threshold function (cf. \eqref{eq:introvae_practice}), and the difficulty in analysing the full optimization objective (cf. \eqref{eq:introvae_theory}). We term our model \textit{Soft-IntroVAE} or S-IntroVAE in short.

Recall that $p_{data}(x)$ denotes the data distribution and $p(z)$ represents some prior distribution over latent variables $z$. The objective functions for the encoder, $E_{\phi}$, and the decoder,  $D_{\theta}$, for samples $x$ and $z$ in our model are given by:
\begin{equation}
\label{eq:sintrovae_loss}
\begin{split}
    \mathcal{L}_{E_{\phi}}(x,z) &= ELBO(x) - \frac{1}{\alpha}\exp(\alpha ELBO(D_{\theta}(z))), \\
    \mathcal{L}_{D_{\theta}}(x,z) &= ELBO(x) + \gamma ELBO(D_{\theta}(z)),
\end{split}
\end{equation}
where $ELBO(x)$ is defined in \eqref{eq:ELBO_def}, $ELBO(D_{\theta}(z))$ is defined similarly but
with $D_{\theta}(z)\sim p_\theta(x|z)$ replacing $x$ in \eqref{eq:ELBO_def},
and $\alpha\geq 0 $ and $\gamma \geq 0$ are hyper-parameters. Note that Eq. \ref{eq:sintrovae_loss} portrays a game between the encoder and the decoder: the encoder is induced to distinguish, through the ELBO value, between real samples (high ELBO) and generated samples (low ELBO), while the decoder is induced to generate samples the `fool' the encoder. The complete S-IntroVAE objective takes an expectation of the losses above over real and generated samples:
\begin{equation}
\label{eq:sintrovae_loss_expectation}
\begin{split}
    \mathcal{L}_{E_{\phi}} &= \mathbb{E}_{x\sim p_{data}, z\sim p(z)}\left[ \mathcal{L}_{E_{\phi}}(x,z) \right],\\
    \mathcal{L}_{D_{\theta}} &= \mathbb{E}_{x\sim p_{data}, z\sim p(z)}\left[ \mathcal{L}_{D_{\theta}}(x,z)\right].
\end{split}
\end{equation}
When the posterior is Gaussian, the losses in \eqref{eq:sintrovae_loss_expectation} can be optimized effectively by SGD using the reparametrization trick.

There are two key differences between S-IntroVAE and IntroVAE. The first is that we utilize the complete ELBO term instead of just the KL, which will allow us to provide a complete variational inference-based analysis. The second difference is replacing the hard threshold in Eq. \ref{eq:introvae_practice} with a soft exponential function over the ELBO, henceforth denoted \textit{expELBO}. The effect of both of these functions is similar -- they induce a separation between the posterior distributions over latent variables of real samples and generated ones. However, as we report in the sequel, the soft threshold in Eq. \ref{eq:sintrovae_loss} is much easier to optimize, and results in improved training stability. 

At this point, the reader may question what minimizing the ELBO for generated samples (through the expELBO in \eqref{eq:sintrovae_loss}) means, as minimizing a lower bound on the log-likelihood does not imply that the likelihood of generated samples decreases~\cite{daniel2019deep}. In addition, for a decoder that produces near-perfect generated samples, it may seem that the expELBO term seeks to reduce sample quality. In the following, we answer these questions by analyzing the equilibrium of \eqref{eq:sintrovae_loss_expectation}.

\subsection{Analysis}
In this section we analyze the Nash equilibrium of the game in \eqref{eq:sintrovae_loss_expectation}. We consider a non-parametric setting, where the encoder and decoder can represent any distribution. This is a typical setting for analyzing adversarial generative models \cite{goodfellow2014generative,introvae18}. For simplicity, we focus on discrete distributions, but our analysis easily extends to the continuous case. Also, to ease the presentation, we focus on the case $\alpha=1$. The analysis for general $\alpha$ provides similar insights and is provided in the supplementary material, along with detailed proofs.

We introduce the following notation. The encoder is represented by the approximate posterior distribution $q \doteq q(z|x)$. The decoder is represented using $d \doteq p_d(x|z)$. These are the controllable distributions in our generative model. The latent prior is denoted $p(z)$ and is not controlled. Slightly abusing notation, we also denote $p_d(x) = \mathbb{E}_{p(z)}[p_d(x|z)]$ as the distribution of generated samples. For some distribution $p(x)$, let  $H(p) = -\mathbb{E}\left[ \log p(x) \right]$ denote its Shannon entropy.  

We define $d^*$ as follows:
\begin{equation}\label{eq:dstar_def}
    d^* \in \argmin_{d} \left\{KL(p_{data} \| p_{d}) + \gamma H(p_{d}(x)) \right\}.
\end{equation}
Note that for $\gamma=0$, we have that $p_{d^*} = p_{data}$. For $\gamma>0$, however, $p_{d^*}$ represents a balance between closeness to $p_{data}$ and low entropy. 
We make the following assumption.
\begin{assumption}\label{ass:ass_1}
For all $x$ such that $p_{data}( x)>0$ we have that $p_{d^*}(x) \leq \sqrt{p_{data}(x)}$.
\end{assumption}
Assumption \ref{ass:ass_1} can be seen as a condition of the closeness between $p_{d^*}$ and $p_{data}$, and essentially requires that the effect of the entropy minimization term in \eqref{eq:dstar_def} is limited. Intuitively, if $\gamma$ is small enough, we should always be able to satisfy Assumption \ref{ass:ass_1}. This is established in the next result.
\begin{proposition}\label{prp:prp_1}
For any $p_{data}$, there exists $\gamma>0$ such that $p_{d^*}$, as defined in \eqref{eq:dstar_def}, satisfies Assumption \ref{ass:ass_1}.
\end{proposition}

We are now ready to state our main result -- that $p_{d^*}$ is an equilibrium point of the S-IntroVAE model.

\begin{theorem}\label{thm:equilibrium}
Let $d^*$ be defined as in \eqref{eq:dstar_def}. Denote $q^* = p_{d^*}(z|x)$. If Assumption \ref{ass:ass_1} holds, then $\left(q^*,  {d^*}\right)$ is a Nash equilibrium of \eqref{eq:sintrovae_loss_expectation}. 
\end{theorem}

Interestingly, Theorem \ref{thm:equilibrium} shows that the S-IntroVAE model does not converge to the data distribution, but to an entropy-regularized version of it. One should question the effect of such regularization, in light of the typical goal of generating samples that are similar to the data distribution. The experiments in Section \ref{sec:experiments}, on various 2-dimensional datasets, illustrate that S-IntroVAE learns distributions with sharper supports than a standard VAE, but without negative effects such as mode dropping. Our image experiments further support this statement.

We now contrast our analysis with the analysis of \cite{introvae18}. First, we note that ignoring the reconstruction terms in the analysis, as done by \cite{introvae18}, leads to significantly different insights. For example, if one removes the $ELBO(x)$ term from $\mathcal{L}_{D_\theta}$, our analysis shows that the model will effectively only minimize $H(p_{d}(x))$, without any dependence on $p_{data}$ (see Appendix \ref{apndx:sec_theory} for more details). Indeed, the empirical results in \cite{introvae18} were only obtained with the $ELBO(x)$ term in $\mathcal{L}_{D_\theta}$.
Furthermore, our analysis, as detailed in the supplementary material, does not build on representing our model as a particular instance of the EBGAN, but explicitly builds on the variational properties of the ELBO. In this sense, our results more closely tie between variational inference principles and adversarial generative models. 

It is important to note that at equilibrium, the encoder converges to the true posterior $q^* = p_{d^*}(z|x)$. Thus, the exponential penalty in $\mathcal{L}_E$ does not harm the inference properties of the encoder. This important conclusion does not appear in the analysis of \cite{introvae18}.

We finally remark on the parameter $\gamma$. While the analysis makes a strict assumption on $\gamma$, as evident in Assumption \ref{ass:ass_1} and Proposition \ref{prp:prp_1}, in all our experiments we set $\gamma=1$. Thus, it seems that in practice the requirements for obtaining decent results are much less stringent.

\section{Implementation}
\label{sec:impl}
In this section, we outline several implementation modifications of the S-IntroVAE algorithm that proved helpful in practice. Pseudo-code of our algorithm is depicted in Algorithm \ref{alg:training_short}, and a more detailed version can be found in Appendix \ref{apndx:algo}. In all our experiments, we set $\alpha=2$ (other values, such as $\alpha=1$, work similarly) and $\gamma=1$ (cf. \eqref{eq:sintrovae_loss}). Moreover, for the ELBO terms, we use the $\beta$-VAE~\cite{higgins2017beta} formulation and rewrite it as:  $ELBO(x) = \beta_{rec}\mathbb{E}_{q_{\phi}(z\mid x)}\left[\log p_{\theta}(x \mid z)\right] -\beta_{kl}KL\left[q_{\phi}(z\mid x) \mid \mid p(z) \right]$.
The hyperparameters $\beta_{rec}$ and $\beta_{kl}$ control the balance between inference and sampling quality respectively.\footnote{In principle, only the ratio between $\beta_{rec}$ and $\beta_{kl}$ affects the loss function. However, we found it easier to work with two parameters, as the balance between ELBO and expELBO is affected by both parameters. } When $\beta_{rec} > \beta_{kl}$, the optimization is focused on good reconstructions, which may lead to less variability in the generated samples, as latent posteriors are allowed to be very different from the prior, and when $\beta_{rec} < \beta_{kl}$, there will be more varied samples, but reconstruction quality will degrade. Note that when $\beta_{rec} << \beta_{kl}$, the VAE is prone to \textit{posterior collapse}, and we further discuss this issue in Appendix \ref{apndx:collapse}. For most cases, we found that values between 0.05 and 1 for $\beta_{rec}$ and $\beta_{kl}$ work well, and the best configuration depends on the architecture and data set, as we detail in Section \ref{sec:experiments}.

Each ELBO  term in \eqref{eq:sintrovae_loss} can be considered as an instance of $\beta$-VAE and can have different $\beta_{rec}$ and $\beta_{kl}$ parameters. However, we set them all to be the same, except for the ELBO inside the exponent. For this term, $\beta_{kl}$ controls the repulsion force of the posterior for generated samples from the prior. 
We found that good results are obtained when $\beta_{kl}$ is set to be on the order of magnitude of $z_{dim}$ (e.g., for $z_{dim} = 128$, $\beta_{kl}$ should be set to values around 128).
In Algorithm \ref{alg:training_short}, this specific parameter is denoted as $\beta_{neg}$.
Moreover, notice that the decoder tries to minimize the reconstruction error for generated data, which may slow down convergence, as at the beginning of the optimization the generated samples are of low quality. Thus, we introduce hyperparameter $\gamma_{r}$ (not to be confused with $\gamma$ from \eqref{eq:sintrovae_loss}) that multiplies only the reconstruction term of the generated data in the ELBO term of the decoder in \eqref{eq:sintrovae_loss}. 
In all our experiments we used a constant $\gamma_r=1e-8$ independently of the data type. We remark that setting $\gamma_r$ to zero had a significant detrimental effect on performance. Finally, we use a scaling constant $s$ to balance between the ELBO and the expELBO terms in the loss, and we set $s$ to be the inverse of the input dimensions. This scaling constant prevents the expELBO from vanishing for high-dimensional input. Our code is publicly available at \url{https://github.com/taldatech/soft-intro-vae-pytorch}.

\begin{algorithm}[t]
\caption{Training Soft-IntroVAE (pseudo-code)}
\label{alg:training_short}
\begin{algorithmic}[1]
\Require $\beta_{rec}, \beta_{kl}, \beta_{neg}, \gamma_r$
\State \(\phi_E, \theta_D \gets\) Initialize~ network~ parameters
\State $s \gets 1 / \text{input dim}$ \Comment{Scaling constant}
\While{not converged}
\State \(X \gets\) Random mini-batch from dataset
\State \(Z \gets\) \(E(X)\) \Comment{Encode}
\State \(Z_f \gets\) Samples from prior \(N(0,I)\)
\Procedure {UpdateEncoder}{$\phi_E$}
    \State \(X_r \gets\) \(D(Z)\), \(X_f \gets\) \(D(Z_f)\) \Comment{Decode}
    \State \(Z_{ff} \gets\) \(E(X_f)\)
    \State\(X_{ff} \gets\) \(D(Z_{ff})\)
    \State $\text{ELBO} \gets s \cdot ELBO(\beta_{rec}, \beta_{kl}, X, X_r, Z) $
    \State $\text{ELBO}_f \gets ELBO(\beta_{rec}, \beta_{neg}, X_f, X_{ff}, Z_{ff}) $
    \State $\text{expELBO}_f \gets 0.5\exp(2s \cdot \text{ELBO}_f)$
    \State $L_E \gets \text{ELBO} - \text{expELBO}_f $ \Comment{Eq. \ref{eq:sintrovae_loss}}
    \State \( \phi_E \gets \phi_E + \eta\nabla_{\phi_E} ( L_E ) \) \Comment{Adam update}
\EndProcedure
\Procedure {UpdateDecoder}{$\theta_D$}
    \State \(X_r \gets\) \(D(Z)\), \(X_f \gets\) \(D(Z_f)\) \Comment{Decode}
    \State \(Z_{ff} \gets\) \(E(X_f)\)
    \State \(X_{ff} \gets\) \(sg(D(Z_{ff}))\) \Comment{\textit{sg: stop-gradient}}
    \State $\text{ELBO} \gets \beta_{rec}L_{rec}(X, X_r) $
    \State $\text{ELBO}_f \! \gets \!\! ELBO(\gamma_r \cdot \beta_{rec}, \beta_{kl}, X_f, X_{ff}, Z_{ff}) $
    \State $L_D \gets s \cdot (\text{ELBO}  + \text{ELBO}_f) $ \Comment{Eq. \ref{eq:sintrovae_loss}}
    \State \(\theta_D \gets \theta_D + \eta\nabla_{\theta_D} (L_D ) \)\Comment{Adam update}
\EndProcedure
\EndWhile
\end{algorithmic}
\end{algorithm}

\section{Experiments}\label{sec:experiments}

In this section, we detail our experiments with the following goals in mind: (1) Understanding the distributions that S-IntroVAE learns to generate; (2) Evaluating the robustness and stability of S-IntroVAE compared to IntroVAE; (3) Benchmarking S-IntroVAE on high-quality image synthesis; (4) Demonstrating a practical application of S-IntroVAE to unsupervised image translation; and (5) Evaluating our model's capability of accurate likelihood prediction for the task of OOD detection in images.

To answer (1) and (2), we investigate learning of 2-dimensional distributions, which are both easy to interpret, and enable a quantitative evaluation of quality and robustness. For (3), (4), and (5), we experiment with standard image data sets.

\subsection{2D Toy Datasets}
\label{2d_exp}

We evaluate our method on four 2D datasets: 8 Gaussians, Spiral, Checkerboard and Rings \cite{bnaf19, grathwohl2018ffjord}, and compare with a standard VAE and IntroVAE. 

We calculate 3 metrics to measure how well the methods learn the true data distribution: KL-divergence and Jensen–Shannon-divergence (JSD), and a custom metric, \textit{grid-normalized ELBO} (gnELBO).
The KL and JSD are calculated using empirical histograms of samples generated from the trained models and samples generated from the real data distribution. Thus, these metrics evaluate the generation capabilities of the learned model. The grid-normalized ELBO, on the other hand, treats the ELBO as an energy term. We normalize the ELBO of the learned model over a grid of points to produce a normalized energy term  $\hat{ELBO}$, and we measure $\mathbb{E}_{x\sim p_{data}}\left[ -\hat{ELBO}(x)\right]$ by sampling from the data distribution. This metric is affected by the encoder and thus effectively measures the inference capabilities of the model -- a good model should assign a high ELBO for likely samples and a low ELBO for unlikely ones. For all metrics, \textit{lower is better}. 

In Figure \ref{fig:2d_plots} we plot random samples from the models and a density estimation, obtained by approximating $p(x)$ with $\exp(ELBO)$. In order to tune the algorithms for each dataset, we ran an extensive hyperparameter grid search of 81 runs for the standard VAE, 210 runs for S-IntroVAE, and 1260 runs for IntroVAE, due to the additional $m$ parameter. The architecture for all methods is a simple 3-layer fully-connected network with 256 hidden units and ReLU activations, and the latent space dimension is 2. The complete set of hyperparameters and their range is provided in Appendix \ref{apndx:hyper_2d}. 

Results for the different evaluation metrics are shown in Table \ref{tab:toy_exp}. 
Evidently, Soft-IntroVAE outperforms IntroVAE both quantitatively and qualitatively, and both are superior to the standard VAE. Note that the standard VAE assigns low energy (high likelihood) to points outside the data support, as pointed out in Section \ref{sec:background}. The adversarial loss in the introspective models, on the other hand, prohibits the model from generating such samples.

\begin{figure}
     \centering
     \begin{subfigure}[b]{0.4\textwidth}
         \centering
         \includegraphics[width=\textwidth]{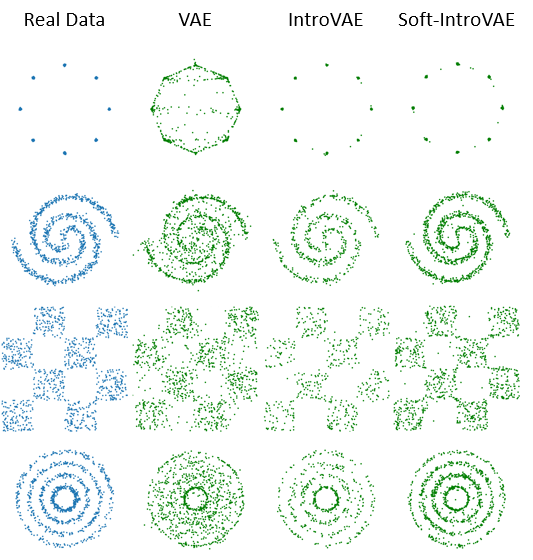} 
         \caption{Samples from the trained models.
         }
         \label{2d_samples}
     \end{subfigure}
     \hfill
     \begin{subfigure}[b]{0.4\textwidth}
         \centering
         \includegraphics[width=\textwidth]{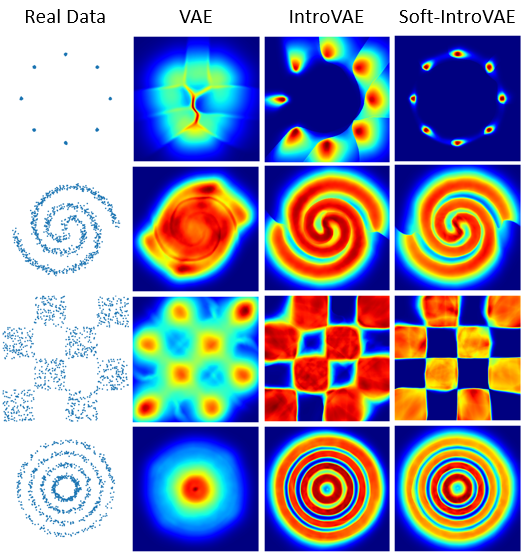}
         \caption{Density estimation with the trained models.}
         \label{2d_desnsity}
     \end{subfigure}
        \caption{Unsupervised learning of 2D datasets.
        }
        \label{fig:2d_plots}
\end{figure}

\begin{table}[ht]
\begin{center}
\begin{scriptsize}
\begin{tabular}{lccccc}
    & & VAE & IntroVAE & Soft-IntroVAE \\
    \hline
    \multirow{3}{*}{8 Gaussians}
    &gnELBO & 7.42$\pm$ 0.07 & 1.29$\pm$0.76 & \textbf{1.25$\pm$0.35}  &\\
    &KL & 6.72$\pm$0.46 & 2.53$\pm$1.07 & \textbf{1.25$\pm$0.11} & \\
    &JSD & 16.04$\pm$0.3 & 1.67$\pm$0.46 & \textbf{0.96$\pm$0.15} &\\
    \hline
    \multirow{3}{*}{Spiral}
    &gnELBO & 6.19$\pm$0.06 & 5.87$\pm$0.03 & \textbf{5.21$\pm$0.04} &\\
    &KL & 9.8$\pm$0.48 & 8.38$\pm$0.45 & \textbf{8.13$\pm$0.3} & \\
    &JSD & 4.89$\pm$0.05 & 3.58$\pm$0.04 & \textbf{3.37$\pm$0.04} & \\
    \hline
    \multirow{3}{*}{Checkerboard}
    &gnELBO & 8.53$\pm$0.1  & 8.54$\pm$0.11 & \textbf{4.47$\pm$0.29} &\\
    &KL & 20.91$\pm$0.45 & \textbf{19.03$\pm$0.34} & 20.27$\pm$0.21 &\\
    &JSD & 9.78$\pm$0.04 & 9.07$\pm$0.1 & \textbf{9.06$\pm$0.15} &\\
    \hline
    \multirow{3}{*}{Rings}
    &gnELBO & 6.4$\pm$0.04 & 7.25$\pm$0.18 & \textbf{6.3$\pm$0.08} &\\
    &KL & 13.16$\pm$0.55 & 10.21$\pm$0.49 & \textbf{9.18$\pm$0.33} &\\
    &JSD & 7.26$\pm$0.07 & 4.24$\pm$0.11 & \textbf{4.13$\pm$0.09} &\\
    \hline
\end{tabular}
\end{scriptsize}
\end{center}
\caption{Results on 2D datasets. 
Grid-normalized ELBO is in $1e^{-7}$ units. Results are averaged over 5 seeds.}
\label{tab:toy_exp}
\end{table}

\subsection{Training Stability of Soft-IntroVAE}
In practice, we found that training the original IntroVAE model was very difficult, and prone to instability. In fact, we were not able to reproduce the results reported in \cite{introvae18} on image datasets, even when using the authors' published code~\cite{introvae18code} or other implementations, and countless parameter investigations. We suspect that the algorithm is very sensitive to the choice of the $m$ parameter: at some point during training when the KL of generated samples is larger than the threshold $m$, there is no more adversarial signal for the encoder (as there are no gradients from the KL-divergence of generated samples), while the decoder still tries to `fool' the encoder. Finding the right $m$ where both the encoder and decoder can maintain the adversarial signal and reach equilibrium is the key for IntroVAE success.

To demonstrate that the $m$ hyperparameter in IntroVAE plays an important role in the stability and convergence of IntroVAE, we pick the best combination of hyperparameters from the extensive search we performed for the experiments in Section \ref{2d_exp} on the 8-Gaussians 2D dataset, and plot the KL and JSD (lower is better) for varying values of $m$. 
For comparison, we plot the results of S-IntroVAE with the same hyperparameters (without $m$) in Figure \ref{fig:m_compare}. Evidently, IntroVAE is very sensitive to this hyperparameter, while our method does not require it. We remark that for different datasets, significantly different values of $m$ were required to obtain reasonable results. We found the soft expELBO term less sensitive, and we obtained good results on a wide range of values as described in Section \ref{sec:impl}.

\begin{figure}[t]
    \centering
    \includegraphics[width=0.4\textwidth]{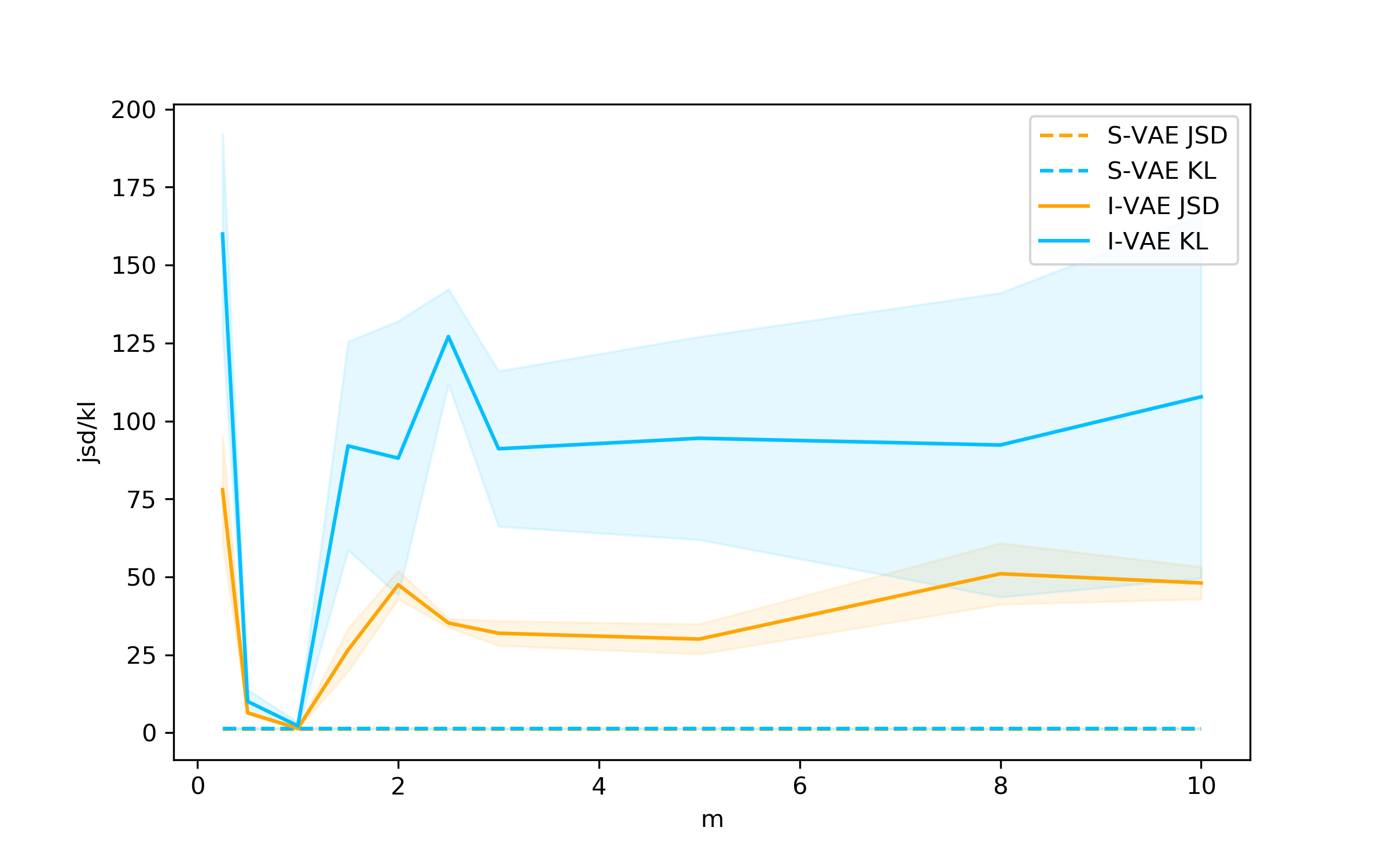} 
    \vspace{-0.5em}
    \caption{Stability investigation. KL divergence and JSD metrics for IntroVAE (I-VAE in the plot) with respect to $m$ (see text for details). 
    S-IntroVAE (S-VAE in the plot) does not require the $m$ parameter, and we plot its performance with all other hyper-parameters the same as IntroVAE. Results are averaged over 5 seeds. Note that the constant lines of the KL and JSD of S-VAE overlap.}
    \label{fig:m_compare}
    \vspace{-1.5em}
\end{figure}

\subsection{Image Generation}
\label{img_exp}
In this section, we evaluate Soft-IntroVAE on image synthesis in terms of both inference (i.e., reconstruction) and sampling (generation). To measure the quality and variability of the images, we report the Fréchet inception distance (FID) based on 50,000 generated images from the model, and use the training samples as the reference images \cite{karras2019style, pidhorskyi2020adversarial}. Detailed hyperparameter settings and data set details are provided in the supplementary material.

\textbf{Architectures and Hyperparameters:} 
We experiment with two convolution-based architectures: (1) IntroVAE's~\cite{introvae18} encoder-decoder architecture with residual-based convolutional layers (see Figure \ref{fig:vae_arch}) and (2) ALAE's \cite{pidhorskyi2020adversarial} style-based autoencoder architecture, which adopted StyleGAN's \cite{karras2020analyzing} style generator to a style-based encoder (see Figure \ref{fig:style_vae_arch}). For the style-based architecture, we also use progressive growing as in \cite{karras2017progressive, karras2019style, pidhorskyi2020adversarial}, where we start from low-resolution $4\times 4$ images and progressively increase the resolution by smoothly blending in new blocks in the encoder and decoder. The reconstruction loss is chosen to be the pixel-wise mean square error (MSE). 
For the detailed description of the architectures and hyperparameters, see Appendix \ref{apndx:hyper_image}.

\begin{figure*}
     \centering
    
     \begin{subfigure}[b]{0.25\textwidth}
             \centering
             \includegraphics[width=2cm,height=5.5cm,keepaspectratio]{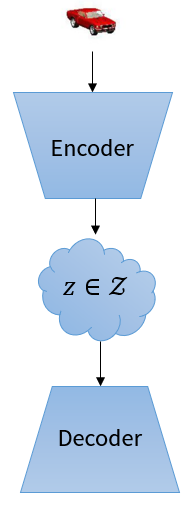}
             \caption{Standard Architecture}
             \label{fig:vae_arch}
    \end{subfigure}
    \hspace{-1.0em}
     \begin{subfigure}[b]{0.25\textwidth}
             \centering
             \includegraphics[width=3cm,height=10cm,keepaspectratio]{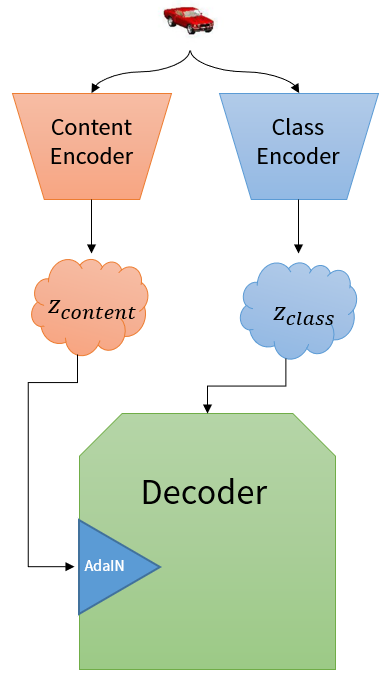}
             \caption{Disentanglement Architecture}
             \label{fig:disentangle_vae_arch}
    \end{subfigure}
    \begin{subfigure}[b]{0.5\textwidth}
         \centering
         \includegraphics[width=8cm,height=12cm,keepaspectratio]{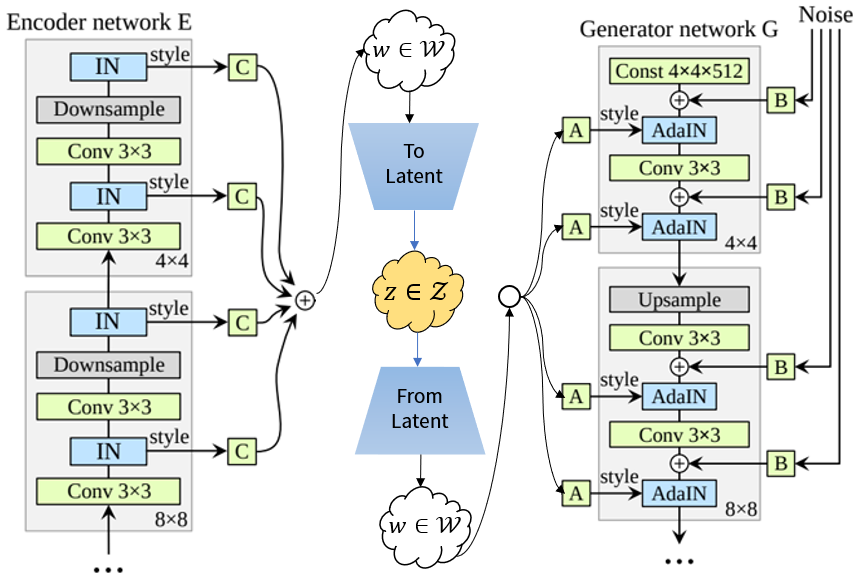}
         \caption{Style-based Architecture inspired by \cite{pidhorskyi2020adversarial}}
         \label{fig:style_vae_arch}
    \end{subfigure}

        \caption{Different architectures used in our experiments. (a) Standard MLP/CNN encoder-decoder architectures; (b) Based on the architecture proposed by \cite{gabbay2019demystifying}: separate encoders with different latent spaces are learned to disentangle class from content. The decoder uses adaptive Instance Normalization to account for the class latent variable when decoding the content; (c) Style-based architecture proposed by~\cite{pidhorskyi2020adversarial}: the encoder extracts styles which are then mapped to the latent space. The latent variable is then mapped back to styles, which are finally decoded with a style-based decoder.}
        \label{fig:vae_architectures}
\end{figure*}

\textbf{CIFAR-10 dataset:} 
We evaluate both the class-conditional and unconditional settings using the architecture of the original IntroVAE~\cite{introvae18}. For the unconditional setting we report FID of 4.6 and qualitative samples and reconstructions of unseen data are displayed in Figure \ref{fig:cifar_mse}. Evidently, our model is able to generate and reconstruct high-quality samples from the various classes of CIFAR-10. For the class-conditional setting, we report FID of 4.07 and qualitative samples can be found in Appendix \ref{apndx_im_gen} along with more samples from the unconditional model.

\textbf{CelebA-HQ and FFHQ datasets:} 
For both datasets we downscale the images to 256x256 resolution. We use the style-based architecture with a latent and style dimension of 512. 
Our model is capable of generating high-quality samples and faithfully reconstructing unseen samples, as demonstrated in Figure \ref{fig:image_samples_ffhq} and Figure \ref{fig:image_samples_celeb} in the supplementary. We provide more samples in Appendix \ref{apndx_im_gen}. 
Table \ref{tab:celeba_ffhq_fid} quantitatively compares our method's performance to various GANs and explicit density methods. It can be seen that S-IntroVAE outperforms all previous autoencoding-based models, further narrowing the gap between VAEs and GANs.

\begin{center}
    \begin{table}[]
    \begin{tabular}{|l|l|l|l|l|l|}
    \hline
                                                &   CelebA-HQ  & FFHQ \\ \hline
    PGGAN  \cite{karras2017progressive})      & \textbf{8.03} & -   \\ \hline
    BigGAN \cite{brock2018large} & - & 11.48 \\ \hline
    U-Net GAN \cite{schonfeld2020u} & - & \textbf{7.48} \\ \hline \hline
    GLOW    \cite{kingma2018glow}           & 68.93 & -   \\ \hline
    Pioneer   \cite{Heljakka_2019}         & 39.17 & -  \\ \hline
    Balanced Pioneer \cite{Heljakka_2020}   & 25.25 & -   \\ \hline
    StyleALAE       \cite{pidhorskyi2020adversarial}   & 19.21 & - \\ \hline
    SoftIntroVAE (Ours) & \textbf{18.63} &  \textbf{17.55}  \\ \hline
    \end{tabular}
    \caption{Comparison of FID scores (lower is better) for CelebA-HQ and FFHQ datasets at a resolution of 256x256. Note the separation between GANs (top) and explicit density methods (bottom).}
    \label{tab:celeba_ffhq_fid}
\end{table}
\end{center}

\vspace{-3em}
\textbf{Interpolation in the latent space:} 
Figure \ref{fig:celeb_interpolations} shows smooth interpolation between the latent vectors of two images from S-IntroVAE trained on the CelebA-HQ dataset. We provide additional interpolations in Appendix \ref{apndx_interpolation}.

\begin{center}
\begin{figure}
    \includegraphics[width=0.45\textwidth]{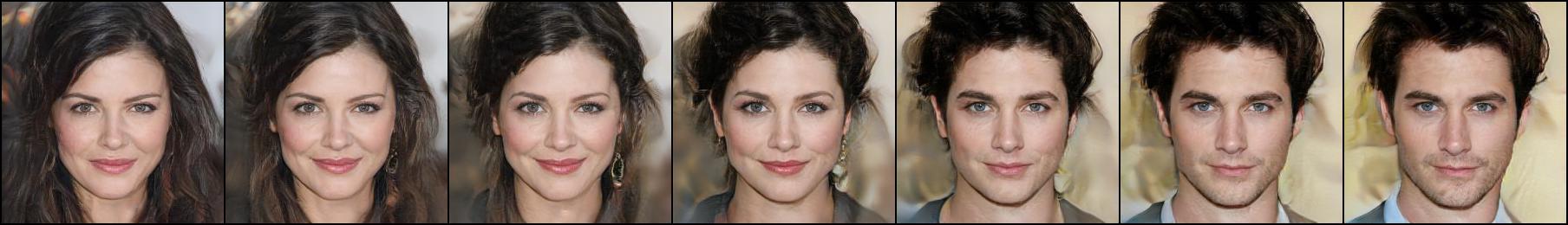}
    \caption{Interpolation in the latent space between two samples from a model trained on CelebA-HQ.}
    \label{fig:celeb_interpolations}
\end{figure}
\end{center}

\begin{figure}
     \centering
     \begin{subfigure}[b]{0.2\textwidth}
             \centering
             \includegraphics[width=\textwidth]{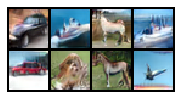}
             \caption{Generated samples (FID: 4.6).}
             \label{fig:cifar10_gen}
        \end{subfigure}
     \hfill
     \begin{subfigure}[b]{0.25\textwidth}
         \centering
         \includegraphics[width=0.8\textwidth]{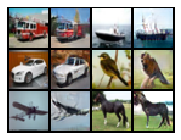}
         \caption{Reconstructions on test data: Left: real, right: reconstruction.}
         \label{fig:cifar10_rec}
     \end{subfigure}
        \caption{Results for the CIFAR-10 dataset in an unconditional setting.}
        \label{fig:cifar_mse}
\end{figure}

\vspace{-3em}
\subsection{Image Translation}
\label{translation_exp}
To demonstrate the advantage of our model's inference capability, we evaluate S-IntroVAE on image translation -- learning disentangled representations for \textit{class} and \textit{content}, and transferring content between classes (e.g., given two images of cars from different visual classes, rotate the first car to be in the angle of the second car, without altering the car's other visual properties). In a typical \textit{class-supervised} image translation setting, class labels are used to encourage class-content separation. The current SOTA in this setting is LORD~\cite{gabbay2019demystifying}. Our focus here, however, is on the \textit{unsupervised} image translation setting, where no labels are used at any point. A recent study \cite{JMLR:v21:19-976} showed that unsupervised learning of disentangled representations can only work by exploiting some form of inductive bias. Here, we claim that the encoder architecture in LORD effectively adds such a strong inductive bias -- we show that when coupled with our S-IntroVAE training method, we can achieve unsupervised image translation results that come close to the supervised SOTA.

We adopt the two-encoder architecture proposed in LORD, where one encoder is for the class and the other for the content. LORD's decoder uses adaptive Instance Normalization~\cite{huang2017arbitrary} to align the mean and variance of the content features with those of the class features,
as depicted in Figure \ref{fig:disentangle_vae_arch}. 
In our proposed architecture, the ELBO is the sum of the reconstruction error and two KL terms: $ KL\left[q_{\phi_{content}}(z\mid x) \mid \mid p(z) \right] + KL\left[q_{\phi_{class}}(z\mid x) \mid \mid p(z) \right].$
The separation to two encoders imposes strong inductive bias, as the model explicitly learns different representations for the class and content. Following the implementation in \cite{gabbay2019demystifying}, we replace the pixel-wise MSE reconstruction loss with the VGG perceptual loss~\cite{hoshen2019non}. 

We quantitatively evaluate our method on the Cars3D dataset~\cite{reed2015deep}, where the \textit{class} corresponds to the car model and the \textit{content} is the azimuth and elevation. 
We follow \cite{gabbay2019demystifying} and measure content transfer in terms of perceptual similarity by Learned Perceptual Image Patch Similarity (LPIPS)~\cite{zhang2018unreasonable}. Unlike previous methods that use some kind of supervision signal (e.g., class label), our method does not require such signals. As demonstrated qualitatively in Figure \ref{fig:cars_3d} and quantitatively in Table \ref{tab:cars3d_lpips}, our method outperforms most of the supervised methods, and narrows the gap to the SOTA supervised approach. We present additional qualitative results on the KTH dataset~\cite{schuldt2004recognizing} in Figure \ref{fig:kth}. Further details can be found in Appendix \ref{apndx:hyper_translation}.

Interestingly, \cite{gabbay2019demystifying} showed that even with the architecture described above, a standard VAE struggles with learning disentangled representations due to vanishing of the KL term (posterior collapse). Our study shows that when training using our introspective manner, the model is able to overcome this issue, showing an additional benefit of the introspective approach.

\begin{figure}
     \centering
     \begin{subfigure}[b]{0.22\textwidth}
             \centering
             \includegraphics[width=\textwidth]{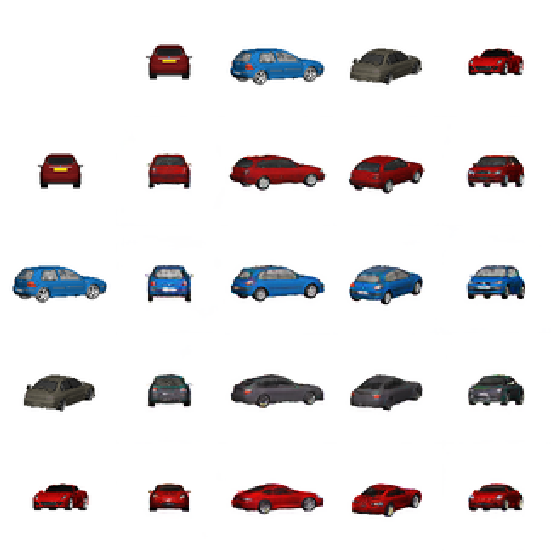}
             \caption{Cars3D}
             \label{fig:cars_3d}
        \end{subfigure}
     \begin{subfigure}[b]{0.22\textwidth}
         \centering
         \includegraphics[width=\textwidth]{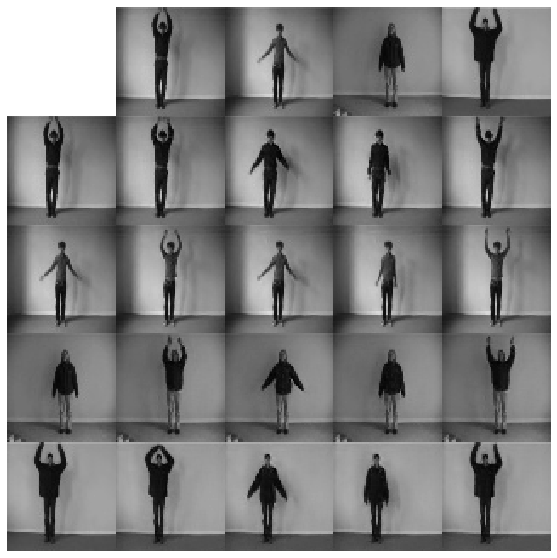}
         \caption{KTH}
         \label{fig:kth}
     \end{subfigure}
        \caption{Qualitative results for content transfer on test data from the Cars3D and KTH datasets. The $4\times 4$ bottom-right matrix of images is generated according to a class from images in the left column and content of images in the top row. It is recommended to zoom-in.}
        \label{fig:content_transfer}
\end{figure}

\begin{center}
\begin{table}[]
\begin{tabular}{|l|l|l|}
\hline
Szabó et al. \cite{szabo2017challenges} & Supervised & 0.137 \\ \hline
ML-VAE  \cite{bouchacourt2017multi}   & Supervised   & 0.132 \\ \hline
Cycle-VAE \cite{harsh2018disentangling}  & Supervised   & 0.141 \\ \hline
DrNet    \cite{denton2017unsupervised} & Supervised    & 0.095 \\ \hline 
LORD \cite{gabbay2019demystifying}     & Supervised   & \textbf{0.078} \\ \hline \hline
SoftIntroVAE (Ours)      & Unsupervised    &  \textbf{0.084} \\ \hline
\end{tabular}
\caption{Content transfer reconstruction error (LPIPS, lower is better) on the Cars3D dataset.}
\label{tab:cars3d_lpips}
\end{table}
\end{center}

\subsection{Out-of-Distribution (OOD) Detection}
\label{sec:ood}
Finally, we present another application of S-IntroVAE to OOD detection. In OOD detection, it is desired to identify whether a sample $x$ belongs to the data distribution or not. A natural approach is therefore to approximately learn $p_{data}(x)$, using some deep generative model, and then to threshold the likelihood of the sample. Interestingly, Nalisnick et al.~\cite{nalisnick2018deep} recently challenged this approach, with the claim that log-likelihood based models are not effective at OOD detection on image datasets, based on evidence for VAEs and flow-based models. Conversely, we show that an S-IntroVAE model, trained in the same setting of \cite{nalisnick2018deep}, obtains excellent OOD detection results. We estimate the log-likelihood of a data sample $x$ using importance-weighted sampling from the trained models: $\log p(x) \approx \log \sum_{i=1}^M p(x|z_i)\frac{p(z_i)}{q_\phi(z_i |x)} $, where $z_i \sim q_{\phi}(z_i |x)$ and $M$ is the number of Monte-carlo samples (we used $M=5$). In Figure \ref{fig:ood} we plot histograms of log-likelihoods for models trained on CIFAR10, evaluated on train and test data from CIFAR10, and also on data from the SVHN dataset. The SVHN data contains images of street house numbers, and is significantly different from the images classes in CIFAR10, and surprisingly, \cite{nalisnick2018deep} found that a standard VAE assigns higher likelihood to SVHN samples than samples from the original CIFAR10 data. Our results in Figure \ref{fig:ood_vae} indeed confirm this observation. 
However, Figure \ref{fig:ood_sintro_vae} shows that S-IntroVAE correctly assigns significantly higher likelihoods to data from CIFAR10. We provide additional results and analysis in Appendix \ref{sec:apndx_ood}, showing that S-IntroVAE obtains near perfect OOD detection results in all the settings investigated in \cite{nalisnick2018deep}.
We conclude that further research is required to claim that log-likelihood models are generally not effective at OOD detection. In particular, it appears that the architecture and training method of the model can make a significant difference, and our positive OOD results are an encouraging motivation for further research on OOD detection using likelihood based models.

\begin{figure}
     \centering
        \begin{subfigure}[b]{0.4\textwidth}
             \centering
             \includegraphics[width=\textwidth]{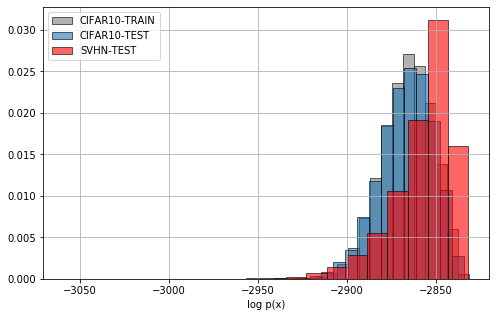}
             \caption{VAE}
             \label{fig:ood_vae}
        \end{subfigure}
     \begin{subfigure}[b]{0.4\textwidth}
         \centering
         \includegraphics[width=\textwidth]{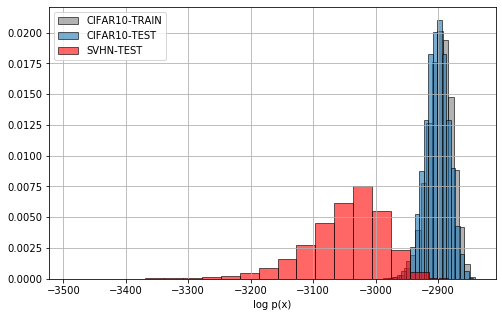}
         \caption{Soft-IntroVAE}
         \label{fig:ood_sintro_vae}
     \end{subfigure}
        \caption{Out-of-distribution detection based on log-likelihood estimation of VAE and Soft-IntroVAE, when the models are trained on CIFAR10 and tested on SVHN.}
        \label{fig:ood}
\end{figure}

\section{Conclusion}
Introspective VAEs narrow the gap between VAEs and GANs in terms of sampling quality, while still enjoying the favorable traits of variational inference models, such as amortized inference. In this work, we proposed the Soft-IntroVAE -- a modification of IntroVAE that is stable to train and simpler to analyze. Our investigation resulted in new insights on introspective training, and our experiments demonstrate competitive image generation results.

We see great potential in introspective models, as they open the door for using high quality generative models in applications that also require fast and high-quality inference. We look forward to future work that will investigate the application of Soft-IntroVAE to domains such as reinforcement learning and novelty detection.
\section{Acknowledgements}
This work is partly funded by the Israel Science Foundation (ISF-759/19) and the Open Philanthropy Project Fund, an advised fund of Silicon Valley Community Foundation.

{\small
\bibliographystyle{ieee_fullname}
\bibliography{egbib}
}
\clearpage
\onecolumn
\section{Appendix}
\subsection{Complete Algorithm}
Algorithm \ref{alg:training_complte} depicts the training procedure of Soft-IntroVAE. The difference from Algorithm \ref{alg:training_short} is the additional generated reconstructions, denoted with $X_r$, which are given the same treatment as the 'fake' generated data, denoted with $X_f$. In practice, following \cite{introvae18}, we found it better to consider all generated data from the decoder, reconstructions ($X_r = D(E(x))$) and samples from $p(z)$, as 'fake' samples to speed-up convergence. In Algorithm \ref{alg:training_complte}, $L_{rec}$ is the reconstruction error function (e.g. mean squared error -- MSE) and $KL$ is a function that calculates the KL divergence. The training flow of Soft-IntroVAE is further depicted in Figure \ref{fig:apndx_training_flow}.

\begin{center}
\begin{figure}
    \centering
    \includegraphics[width=0.8\textwidth]{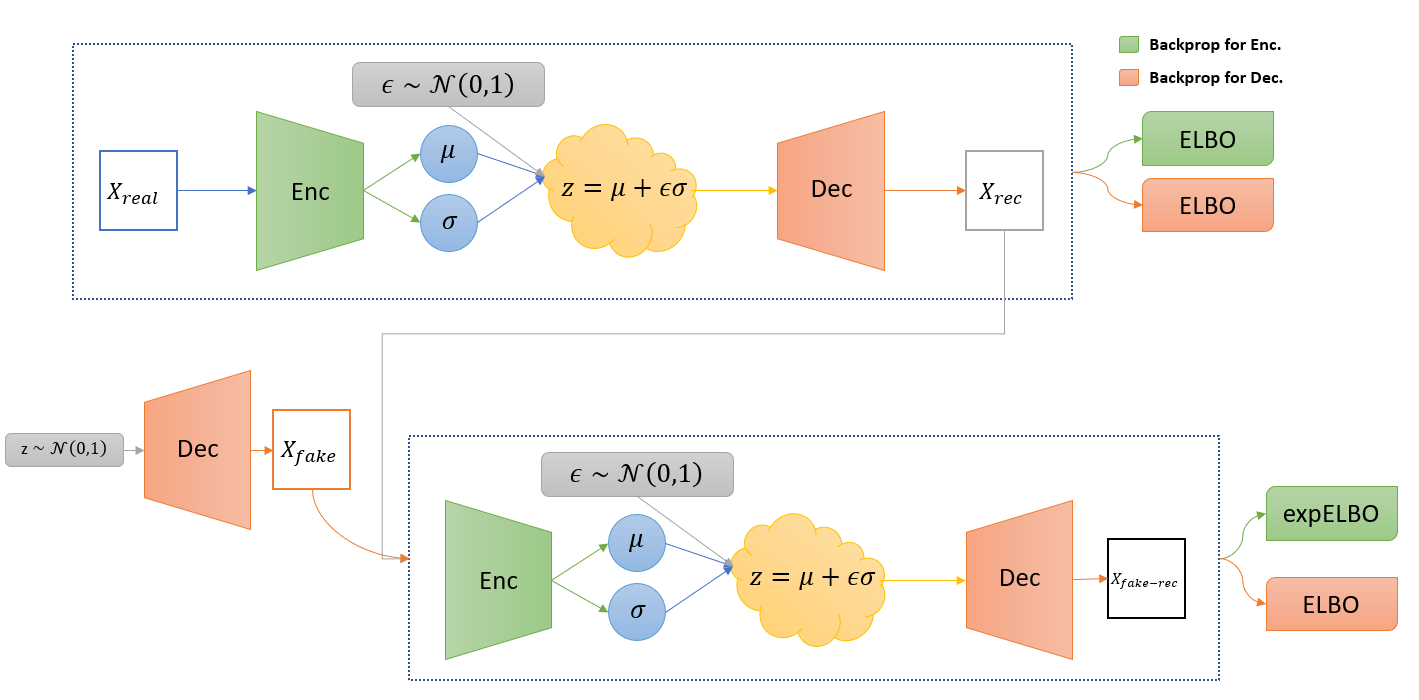}
    \caption{Training flow of Soft-IntroVAE. The ELBO for real samples is optimized for both encoder and decoder, while the encoder also optimizes the expELBO to 'push away' generated samples from the latent space, and the decoder optimizes the ELBO for the generated samples to 'fool' the encoder.}
    \label{fig:apndx_training_flow}
\end{figure}
\end{center}

\label{apndx:algo}
\begin{algorithm}[t]
\caption{Training Soft-IntroVAE}
\label{alg:training_complte}
\begin{algorithmic}[1]
\Require $\beta_{rec}, \beta_{kl}, \beta_{neg}, \gamma_r$
\State \(\phi_E, \theta_D \gets\) Initialize~ network~ parameters
\State $s \gets 1 / \text{input dim}$ \Comment{Scaling constant}
\While{not converged}
\State \(X \gets\) Random mini-batch from dataset
\State \(Z \gets\) \(E(X)\) \Comment{Encode}
\State \(Z_f \gets\) Samples from prior \(N(0,I)\)
\Procedure {UpdateEncoder}{$\phi_E$}
    \State \(X_r \gets\) \(D(Z)\), \(X_f \gets\) \(D(Z_f)\) \Comment{Decode}
    \State \(Z_{rf} \gets\) \(E(X_r)\), \(Z_{ff} \gets\) \(E(X_f)\)
    \State \(X_{rf} \gets\) \(D(Z_{rf})\), \(X_{ff} \gets\) \(D(Z_{ff})\)
    \State $\text{ELBO} \gets s \cdot ELBO(\beta_{rec}, \beta_{kl}, X, X_r, Z) $
    \State $\text{ELBO}_r \gets ELBO(\beta_{rec}, \beta_{neg}, X_r, X_{rf}, Z_{rf}) $
    \State $\text{ELBO}_f \gets  ELBO(\beta_{rec}, \beta_{neg}, X_f, X_{ff}, Z_{ff}) $
    \State $\text{expELBO}_r \gets 0.5\exp(2s \cdot \text{ELBO}_r)$
    \State $\text{expELBO}_f \gets 0.5\exp(2s \cdot \text{ELBO}_f)$
    \State $L_E \gets \text{ELBO} - 0.5\cdot (\text{expELBO}_r + \text{expELBO}_f) $
    \State \( \phi_E \gets \phi_E + \eta\nabla_{\phi_E} ( L_E ) \) \Comment{Adam update (ascend)}
\EndProcedure
\Procedure {UpdateDecoder}{$\theta_D$}
    \State \(X_r \gets\) \(D(Z)\), \(X_f \gets\) \(D(Z_f)\) \Comment{Decode}
    \State \(Z_{rf} \gets\) \(E(X_r)\), \(Z_{ff} \gets\) \(E(X_f)\)
    \State \(X_{rf} \gets\) \(sg(D(Z_{rf}))\), \(X_{ff} \gets\) \(sg(D(Z_{ff}))\)
    \State $\text{ELBO} \gets \beta_{rec}L_{rec}(X, X_r) $
    \State $\text{ELBO}_r \! \gets \!\! ELBO(\gamma_r \cdot\beta_{rec}, \beta_{kl}, X_r, X_{rf}, Z_{rf}) $
    \State $\text{ELBO}_f \! \gets \!\! ELBO(\gamma_r \cdot \beta_{rec}, \beta_{kl}, X_f, X_{ff}, Z_{ff}) $
    \State $L_D \gets s \cdot (\text{ELBO}  + 0.5\cdot (\text{ELBO}_r + \text{ELBO}_f) ) $
    \State \(\theta_D \gets \theta_D + \eta\nabla_{\theta_D} (L_D ) \)\Comment{Adam update (ascend)}
\EndProcedure
\EndWhile
\State
\Function{ELBO}{$\beta_{rec}, \beta_{kl}, X, X_r, Z$}
    \State $ELBO \gets -1 \cdot(\beta_{rec}L_{rec}(X, X_r) +\beta_{kl} KL(Z))$
    \State \textbf{return} $ELBO$
\EndFunction
\end{algorithmic}
\end{algorithm}

\subsubsection{IntroVAE and Soft-IntroVAE Objectives}

\paragraph{Soft-IntroVAE} Expanding S-IntroVAE's objective, which is \textit{minimized}, from Eq. \ref{eq:sintrovae_loss} with the complete set of hyperparameters:
\begin{equation}
\label{eq:sintrovae_loss_complete}
\begin{split}
    \mathcal{L}_{E_{\phi}}(x,z) &= s \cdot(\beta_{rec}\mathcal{L}_r(x) +\beta_{kl}KL(x))
    +\frac{1}{2} \exp(-2s\cdot (\beta_{rec}\mathcal{L}_r(D_{\theta}(z)) + \beta_{neg}KL(D_{\theta}(z)))), \\
    \mathcal{L}_{D_{\theta}}(x,z) &= s \cdot \beta_{rec}\mathcal{L}_r(x) +s \cdot(\beta_{kl}KL(D_{\theta}(z)) +\gamma_r \cdot \beta_{rec}\mathcal{L}_r(D_{\theta}(z))),
\end{split}
\end{equation}
where $\mathcal{L}_r(x) = - \mathbb{E}_{q_{\phi}(z\mid x)}\left[\log p_{\theta}(x \mid z)\right]$ denotes the reconstruction error. 

\paragraph{IntroVAE} Expanding IntroVAE's objective,  which is \textit{minimized}, from Eq. \ref{eq:introvae_practice} with the complete set of hyperparameters:
\begin{equation*}
    \begin{split}
        \mathcal{L}_E(x,z) & = \beta_{rec}\mathcal{L}_r(x) +\beta_{kl}KL(x)
         +\beta_{neg}[m - KL(E(D_{\theta}(z)))]^{+} \\
        \mathcal{L}_D(x,z) & = \beta_{rec}\mathcal{L}_r(x) + \beta_{neg}KL(E(D_{\theta}(z))).
    \end{split}
\end{equation*}

Note that the difference in hyperparameters from S-IntroVAE objectives is the added $m$ hyperparameter for the hard-margin loss in the encoder. In S-IntroVAE, the objectives also include the reconstruction terms for the generated data, where in the decoder they are preceded by $\gamma$ which is set $1e-8$ in all experiments. Also, recall that $s$ is a normalizing constant that is set to the inverse of the input dimensions, and is not required in IntroVAE.

\subsection{Datasets, Architectures and Hyperparameters}
\label{apndx:hyper}

We implement our method in PyTorch~\cite{paszke2017automatic}.  For all experiments, we used the Adam~\cite{adam_14} optimizer with the default parameters, and $\gamma_{r}$ was set to $1e-8$ independently of the dataset. In practice, $\gamma_r$ should be set to a small value.
Note that setting $\gamma_r=0$ can cause a degradation in performance. Experiments with the style-based architecture were run on a machine with 4 Nvidia RTX 2080 GPUs, while the rest used a machine with one GPU of the same type. In what follows, we detail the dataset-specific hyperparameters.

\subsubsection{2D Experiments}
\label{apndx:hyper_2d}
The architecture for all methods is a simple 3-layer fully-connected network with 256 hidden units and ReLU activations, and the latent space dimension is 2. We used a learning rate of $2e-4$, batch size of 512 and ran a total of 30,000 iterations per dataset. We ran an extensive hyperparameter grid search of 81 runs for the standard VAE, 210 runs for S-IntroVAE, and 1260 runs for IntroVAE. The range of the search was $[0.05, 1.0]$ for $\beta_{kl}$ and $\beta_{rec}$, $[\beta_{kl}, 5\beta_{kl}]$ for $\beta_{neg}$ and $[1, 10]$ for $m$. The best combinations of hyperparameters are provided in Table \ref{tab:2d_hyperparameters}.

\begin{table}[ht]
\begin{center}
\begin{scriptsize}
\begin{tabular}{lccccc}
    & & VAE & IntroVAE & Soft-IntroVAE \\
    \hline
    \multirow{3}{*}{8 Gaussians}
    &$\beta_{rec}$ & 0.8 &0.3 & 0.2  &\\
    &$\beta_{kl}$ & 0.05 & 0.5 & 0.3 & \\
    &$\beta_{neg}$ &- & 1.0 & 0.9 &\\
    &m              &- & 1.0 & -   &\\
    \hline
    \multirow{3}{*}{Spiral}
    &$\beta_{rec}$ & 1.0 & 0.2 & 0.2  &\\
    &$\beta_{kl}$ & 0.05 & 0.5 & 0.5 & \\
    &$\beta_{neg}$ &- & 0.5 & 1.0 &\\
    &m              &- & 2.0 & -   &\\
    \hline
    \multirow{3}{*}{Checkerboard}
    &$\beta_{rec}$ & 0.8 & 0.4 & 0.2  &\\
    &$\beta_{kl}$ & 0.1 & 0.2 & 0.1 & \\
    &$\beta_{neg}$ &- & 0.2 & 0.2 &\\
    &m              &- & 8.0 & -   &\\
    \hline
    \multirow{3}{*}{Rings}
    &$\beta_{rec}$ & 0.8 & 0.8 & 0.2  &\\
    &$\beta_{kl}$ & 0.05 & 0.5 & 0.2 & \\
    &$\beta_{neg}$ &- & 0.5 & 1.0 &\\
    &m              &- & 5.0 & -   &\\
    \hline
\end{tabular}
\end{scriptsize}
\end{center}
\caption{Hyperparameters for the 2D datasets.}
\label{tab:2d_hyperparameters}
\end{table}

\subsubsection{Image Generation}
\label{apndx:hyper_image}
CIFAR-10~\cite{Krizhevsky2009LearningML} consists of 60,000 32x32 colour images in 10 classes, with 6,000 images per class. We use the official split of 50,000 training images and 10,000 test images and evaluate in both unconditional and class-conditional settings. We use IntroVAE's \cite{introvae18} architecture\footnote{\url{https://github.com/hhb072/IntroVAE}} with a latent dimension of 128 and train the model for 220 epochs with a learning rate of $2e-4$ and batch size of 32. The hyperparameters are $\beta_{rec}=\beta_{kl}=1$ and $\beta_{neg}=256$.
IntroVAE's encoder-decoder general architecture with residual-based convolutional layers for images at resolution of 1024$\times$1024 is depicted in Table \ref{tab:introvae_arch}. For CIFAR-10 we used 3 residual blocks in both encoder and decoder with channels (64, 128, 256).

\begin{table*}[t]
    \begin{tabular}{|lcc|}
    \hline
    \textbf{Encoder} & Act. & Output shape  \\
    \hline
    Input image & -- & $\makebox[\widthof{512}][c]{3}\times\makebox[\widthof{512}][c]{1024}\times\makebox[\widthof{512}][c]{1024}$  \\
    Conv  & $5\times5, 16$ & $\makebox[\widthof{512}][c]{16}\times\makebox[\widthof{512}][c]{1024}\times\makebox[\widthof{512}][c]{1024}$  \\
    AvgPool & -- & $\makebox[\widthof{512}][c]{16}\times\makebox[\widthof{512}][c]{512}\times\makebox[\widthof{512}][c]{512}$  \\
    \hline
    Res-block & $\left[ {\begin{array}{*{20}{c}}{1 \times 1,}&{32}\\
    {3 \times 3,}&{32}\\
    {3 \times 3,}&{32}
    \end{array}} \right]$ & $\makebox[\widthof{512}][c]{32}\times\makebox[\widthof{512}][c]{512}\times\makebox[\widthof{512}][c]{512}$  \\
    AvgPool & -- & $\makebox[\widthof{512}][c]{32}\times\makebox[\widthof{512}][c]{256}\times\makebox[\widthof{512}][c]{256}$  \\
    \hline
    Res-block & $\left[ {\begin{array}{*{20}{c}}{1 \times 1,}&{64}\\
    {3 \times 3,}&{64}\\
    {3 \times 3,}&{64}
    \end{array}} \right]$ & $\makebox[\widthof{512}][c]{64}\times\makebox[\widthof{512}][c]{256}\times\makebox[\widthof{512}][c]{256}$  \\
    AvgPool & -- & $\makebox[\widthof{512}][c]{64}\times\makebox[\widthof{512}][c]{128}\times\makebox[\widthof{512}][c]{128}$  \\
    \hline
    Res-block & $\left[ {\begin{array}{*{20}{c}}{1 \times 1,}&{128}\\
    {3 \times 3,}&{128}\\
    {3 \times 3,}&{128}
    \end{array}} \right]$ & $\makebox[\widthof{512}][c]{128}\times\makebox[\widthof{512}][c]{128}\times\makebox[\widthof{512}][c]{128}$  \\
    AvgPool & -- & $\makebox[\widthof{512}][c]{128}\times\makebox[\widthof{512}][c]{64}\times\makebox[\widthof{512}][c]{64}$  \\
    \hline
    Res-block & $\left[ {\begin{array}{*{20}{c}}{1 \times 1,}&{256}\\
    {3 \times 3,}&{256}\\
    {3 \times 3,}&{256}
    \end{array}} \right]$ & $\makebox[\widthof{512}][c]{256}\times\makebox[\widthof{512}][c]{64}\times\makebox[\widthof{512}][c]{64}$  \\
    AvgPool & -- & $\makebox[\widthof{512}][c]{256}\times\makebox[\widthof{512}][c]{32}\times\makebox[\widthof{512}][c]{32}$  \\
    \hline
    Res-block & $\left[ {\begin{array}{*{20}{c}}{1 \times 1,}&{512}\\
    {3 \times 3,}&{512}\\
    {3 \times 3,}&{512}
    \end{array}} \right]$ & $\makebox[\widthof{512}][c]{512}\times\makebox[\widthof{512}][c]{32}\times\makebox[\widthof{512}][c]{32}$  \\
    AvgPool & -- & $\makebox[\widthof{512}][c]{512}\times\makebox[\widthof{512}][c]{16}\times\makebox[\widthof{512}][c]{16}$  \\
    \hline
    Res-block & $\left[ {\begin{array}{*{20}{c}}{1 \times 1,}&{512}\\
    {3 \times 3,}&{512}\\
    {3 \times 3,}&{512}
    \end{array}} \right]$ & $\makebox[\widthof{512}][c]{512}\times\makebox[\widthof{512}][c]{16}\times\makebox[\widthof{512}][c]{16}$  \\
    AvgPool & -- & $\makebox[\widthof{512}][c]{512}\times\makebox[\widthof{512}][c]{8}\times\makebox[\widthof{512}][c]{8}$  \\
    \hline
    Res-block & $\left[ {\begin{array}{*{20}{c}}
    {3 \times 3,}&{512}\\
    {3 \times 3,}&{512}
    \end{array}} \right]$ & $\makebox[\widthof{512}][c]{512}\times\makebox[\widthof{512}][c]{8}\times\makebox[\widthof{512}][c]{8}$  \\
    AvgPool & -- & $\makebox[\widthof{512}][c]{512}\times\makebox[\widthof{512}][c]{4}\times\makebox[\widthof{512}][c]{4}$  \\
    \hline
    Res-block & $\left[ {\begin{array}{*{20}{c}}
    {3 \times 3,}&{512}\\
    {3 \times 3,}&{512}
    \end{array}} \right]$ & $\makebox[\widthof{512}][c]{512}\times\makebox[\widthof{512}][c]{4}\times\makebox[\widthof{512}][c]{4}$  \\
    Reshape & -- & $\makebox[\widthof{512}][c]{8192}\times\makebox[\widthof{512}][c]{1}\times\makebox[\widthof{512}][c]{1}$ \\
    \hline
    FC-1024 & -- & $\makebox[\widthof{512}][c]{1024}\times\makebox[\widthof{512}][c]{1}\times\makebox[\widthof{512}][c]{1}$ \\
    \hline
    Split & -- & 512, 512 \\
    \hline
    \end{tabular}
    \hfill
    \begin{tabular}{|lcc|}
    \hline
    \textbf{Decoder} & Act. & Output shape \\
    \hline
    Latent vector & -- & $\makebox[\widthof{512}][c]{512}\times\makebox[\widthof{512}][c]{1}\times\makebox[\widthof{512}][c]{1}$  \\
    FC-8192 & {\tiny ReLU} & $\makebox[\widthof{512}][c]{8192}\times\makebox[\widthof{512}][c]{1}\times\makebox[\widthof{512}][c]{1}$ \\
    \hline
    Reshape & -- & $\makebox[\widthof{512}][c]{512}\times\makebox[\widthof{512}][c]{4}\times\makebox[\widthof{512}][c]{4}$ \\
    Res-block & $\left[ {\begin{array}{*{20}{c}}{3 \times 3,}&{512}\\
    {3 \times 3,}&{512}
    \end{array}} \right]$ & $\makebox[\widthof{512}][c]{512}\times\makebox[\widthof{512}][c]{4}\times\makebox[\widthof{512}][c]{4}$  \\
    \hline
    Upsample & -- & $\makebox[\widthof{512}][c]{512}\times\makebox[\widthof{512}][c]{8}\times\makebox[\widthof{512}][c]{8}$  \\
    Res-block & $\left[ {\begin{array}{*{20}{c}}{3 \times 3,}&{512}\\
    {3 \times 3,}&{512}
    \end{array}} \right]$ & $\makebox[\widthof{512}][c]{512}\times\makebox[\widthof{512}][c]{8}\times\makebox[\widthof{512}][c]{8}$  \\
    \hline
    Upsample & -- & $\makebox[\widthof{512}][c]{512}\times\makebox[\widthof{512}][c]{16}\times\makebox[\widthof{512}][c]{16}$  \\
    Res-block & $\left[ {\begin{array}{*{20}{c}}{3 \times 3,}&{512}\\
    {3 \times 3,}&{512}
    \end{array}} \right]$ & $\makebox[\widthof{512}][c]{512}\times\makebox[\widthof{512}][c]{16}\times\makebox[\widthof{512}][c]{16}$  \\
    \hline
    Upsample & -- & $\makebox[\widthof{512}][c]{512}\times\makebox[\widthof{512}][c]{32}\times\makebox[\widthof{512}][c]{32}$  \\
    Res-block & $\left[ {\begin{array}{*{20}{c}}{1 \times 1,}&{256}\\
    {3 \times 3,}&{256}\\
    {3 \times 3,}&{256}
    \end{array}} \right]$ & $\makebox[\widthof{512}][c]{256}\times\makebox[\widthof{512}][c]{32}\times\makebox[\widthof{512}][c]{32}$  \\
    \hline
    Upsample & -- & $\makebox[\widthof{512}][c]{256}\times\makebox[\widthof{512}][c]{64}\times\makebox[\widthof{512}][c]{64}$  \\
    Res-block & $\left[ {\begin{array}{*{20}{c}}{1 \times 1,}&{128}\\
    {3 \times 3,}&{128}\\
    {3 \times 3,}&{128}
    \end{array}} \right]$ & $\makebox[\widthof{512}][c]{128}\times\makebox[\widthof{512}][c]{64}\times\makebox[\widthof{512}][c]{64}$  \\
    \hline
    Upsample & -- & $\makebox[\widthof{512}][c]{128}\times\makebox[\widthof{512}][c]{128}\times\makebox[\widthof{512}][c]{128}$  \\
    Res-block & $\left[ {\begin{array}{*{20}{c}}{1 \times 1,}&{64}\\
    {3 \times 3,}&{64}\\
    {3 \times 3,}&{64}
    \end{array}} \right]$ & $\makebox[\widthof{512}][c]{64}\times\makebox[\widthof{512}][c]{128}\times\makebox[\widthof{512}][c]{128}$  \\
    \hline
    Upsample & -- & $\makebox[\widthof{512}][c]{64}\times\makebox[\widthof{512}][c]{256}\times\makebox[\widthof{512}][c]{256}$  \\
    Res-block & $\left[ {\begin{array}{*{20}{c}}{1 \times 1,}&{32}\\
    {3 \times 3,}&{32}\\
    {3 \times 3,}&{32}
    \end{array}} \right]$ & $\makebox[\widthof{512}][c]{32}\times\makebox[\widthof{512}][c]{256}\times\makebox[\widthof{512}][c]{256}$  \\
    \hline
    Upsample & -- & $\makebox[\widthof{512}][c]{32}\times\makebox[\widthof{512}][c]{512}\times\makebox[\widthof{512}][c]{512}$  \\
    Res-block & $\left[ {\begin{array}{*{20}{c}}{1 \times 1,}&{16}\\
    {3 \times 3,}&{16}\\
    {3 \times 3,}&{16}
    \end{array}} \right]$ & $\makebox[\widthof{512}][c]{16}\times\makebox[\widthof{512}][c]{512}\times\makebox[\widthof{512}][c]{512}$  \\
    \hline
    Upsample & -- & $\makebox[\widthof{512}][c]{16}\times\makebox[\widthof{512}][c]{1024}\times\makebox[\widthof{512}][c]{1024}$  \\
    Res-block & $\left[ {\begin{array}{*{20}{c}}
    {3 \times 3,}&{16}\\
    {3 \times 3,}&{16}
    \end{array}} \right]$ & $\makebox[\widthof{512}][c]{16}\times\makebox[\widthof{512}][c]{1024}\times\makebox[\widthof{512}][c]{1024}$  \\
    \hline
    Conv  & $5\times5, 3$ & $\makebox[\widthof{512}][c]{3}\times\makebox[\widthof{1024}][c]{1024}\times\makebox[\widthof{1024}][c]{1024}$  \\
    \hline
    \end{tabular}
\caption{IntroVAE's general architecture for images at resolution $1024 \times 1024$.}
\label{tab:introvae_arch}
\end{table*}

CelebA-HQ~\cite{karras2017progressive} is an improved version of CelebA~\cite{liu2015faceattributes}, and consists of a subset of 30,000 high-quality 1024x1024 images of celebrities, which are split to 29,000 train images and 1,000 test images. FFHQ~\cite{karras2019style} is a high-quality image dataset consisting of 70,000 images of people faces aligned and cropped at resolution of 1024x1024, split to 60,000 train images and 10,000 test images.  

\paragraph{Style-based architecture} The decoder in the style-based architecture borrows the same properties of StyleGAN's \cite{karras2019style} generator, while the encoder is designed after the novel architecture in ALAE~\cite{pidhorskyi2020adversarial}\footnote{\url{https://github.com/podgorskiy/ALAE}}. Every layer in StyleGAN's generator is driven by a style input $w \in \mathcal{W}$, which requires that the encoder will also encode \textit{styles}. Thus, in the style-based architecture the layers in the encoder and decoder are symmetric, such that every layer extracts and injects styles, correspondingly. This is made possible by using Instance Normalization (IN) layers~\cite{huang2017arbitrary}, which provide instance means and standard deviations for every \textit{channel}. 

Mathematically, let $y_i^E$ denote the output of the $i$-th layer in the encoder $E$, the IN module extracts the statistics $\mu(y_i^E)$ and $\sigma(y_i^E)$, representing the style at that level. The second output of the IN module is the normalized version of the input which continues down the pipeline with no more style information. Finally, the latent style variable, $w$, is a weighted sum of the extracted styles: $$ w=\sum_{i=1}^N C_i \begin{bmatrix}\mu(y_i^E) \\ \sigma(y_i^E) \end{bmatrix},$$ where $C_i$'s are learned parameters and $N$ is the number of layers. The style latent variable is then mapped with a fully-connected network to the mean, $\mu_q$, and standard deviation, $\sigma_q$, of the Gaussian latent variable $z \in \mathcal{N}(\mu_q, \sigma_q^2)$, which is done efficiently using the reparameterization trick.

Symmetrically, the latent variable $z$ is mapped back to a style latent variable $w$ using a fully-connected network, which serves as inputs to the Adaptive Instance Normalization (AdaIN) layers~\cite{huang2017arbitrary} in the decoder. The complete style-based architecture is depicted in Figure \ref{fig:style_vae_arch}. In our experiments the latent variables $z$ and $w$ have the same dimensions of 512, the mapping from style to latent in $E$ has 3 layers, while the mapping from latent to style in $D$ has 8 layers, both with 512 hidden units in each layer.

The training using the style-based architecture is done is a progressive growing fashion, similar to \cite{karras2017progressive, karras2019style, pidhorskyi2020adversarial}, where we start from low resolution (4$\times$4 pixels) and progressively increase the resolution by smoothly blending in new blocks to $E$ and $D$. For CelebA-HQ $\beta_{rec}=0.05$, and for FFHQ $\beta_{rec}=0.1$, while for both datasets $\beta_{kl}=0.2$ and $\beta_{neg}=512$, and the maximal learning rate is $1.5e-3$. The CelebA-HQ model is trained for 230 epochs and the FFHQ model for 270 epochs, where the training reaches the 256$\times$256 resolution at epoch 180 (30 epoch per resolution until 256$\times$256).

\subsection{Image Translation}
\label{apndx:hyper_translation}
For the image translation experiments, we use the architecture proposed in LORD \cite{gabbay2019demystifying}, and use two encoders, one for the class and one for the content, where the latent representation of the class controls the adaptive Instance Normalization \cite{huang2017arbitrary} of the latent representation of the content in the decoder. More specifically, the encoder is composed of convolutional layers with channels (64, 128, 256, 256), followed by 3 fully-connected layers to produce the parameters of the Gaussian latent variable. The decoder consists of 3 fully-connected layers followed by 6 convolutional layers, where the first 4 are preceded by an upsampling layer and followed by AdaIN normalization, that uses the latent representation of the class. All layers are activated with LeakyReLU.
This architecture is depicted in Figure \ref{fig:disentangle_vae_arch}.

For this specific choice of architecture, note that the KL terms update each encoder separately, the reconstruction term jointly updates both encoders, as it is a function of the latents from both encoders. 

Similar to \cite{gabbay2019demystifying}, all images are resized to 64$\times$64 resolution and we set the latent dimension of the class to be 256, and 128 for the content. In all experiments we used Adam optimizer with a learning rate of $1e-4$, batch size of 64 and ran a total of 400 train epochs.

\paragraph{Cars3D dataset}  The Cars3D dataset \cite{reed2015deep} consists of 183 car CAD models, labelled with 24 azimuth directions and 4 elevations. For this dataset, the class is considered to be the car model and the azimuth and elevation as the content. We use 163 car models for training and the other 20 are held out for testing. As Cars3D includes ground-truth labels, we are able to test the quality of disentanglement using the same evaluation procedure as in \cite{gabbay2019demystifying}, by measuring the content transfer reconstruction loss. As suggested by \cite{gabbay2019demystifying}, we replace the pixel-wise MSE reconstruction loss with the VGG perceptual loss as implemented by \cite{hoshen2019non}. The hyperparameters used for this dataset: $\beta_{kl}^{content}=\beta_{kl}^{class}=1.0, \beta_{rec}=0.5, \beta_{neg}^{content}=2048$ and $\beta_{kl}^{class}=1024$.

\paragraph{KTH dataset} We further evaluate on the KTH dataset \cite{schuldt2004recognizing} which contains videos of 25 people performing different activities. For training our model, we extract grayscale image frames from all of the videos. As there are no ground-truth labels, we \textit{assume} the class is the person identity and the content is other unlabeled transitory attributes such as skeleton position. Similarly to \cite{gabbay2019demystifying}, due to the very limited amount of subjects, we use all the identities
for training, holding out 10\% of the images for testing. Moreover, we found that using MSE pixel-wise loss worked better for the grayscale images than the VGG perceptual loss. The hyperparameters used for this dataset: $\beta_{kl}^{content}=\beta_{kl}^{class}=0.5, \beta_{rec}=1.0, \beta_{neg}^{content}=2048$ and $\beta_{kl}^{class}=1024$.

\subsection{Additional Results}
\label{apndx_results}
In this section, we provide additional results from the experiments we described. 

\subsubsection{Image Generation and Reconstruction}
\label{apndx_im_gen}
\paragraph{CIFAR-10 dataset} We trained two types of models: (1) unconditional model and (2) class-conditional model. In Figure \ref{fig:apndx_cifar_samples} we present random (i.e., no cherry-picking) samples from a trained unconditional model (FID: 4.6), and Figure \ref{fig:apndx_cifar_recons} presents random reconstructions from the test set. For the conditional model, we used a one-hot vector representation for the class, and trained a conditional VAE (CVAE) using Soft-IntroVAE's objectives. Random samples from the class-conditional model can be seen in Figure \ref{fig:apndx_cifar_samples_cond} (FID: 4.07), and random reconstructions from the test set in Figure \ref{fig:apndx_cifar_recons_cond}. It can be seen that when including a supervision signal (class labels), the samples tend to be slightly more structured, which is also reflected in the FID score.

\begin{figure}
     \centering
        \begin{subfigure}[b]{0.4\textwidth}
             \centering
             \includegraphics[width=\textwidth]{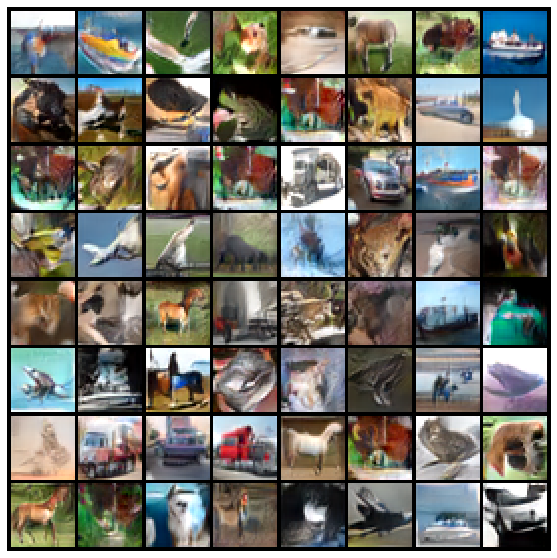}
             \caption{Generated samples (FID: 4.6).}
             \label{fig:apndx_cifar_samples}
        \end{subfigure}
     \begin{subfigure}[b]{0.4\textwidth}
         \centering
         \includegraphics[width=\textwidth]{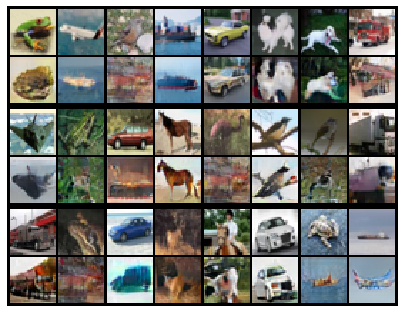}
         \caption{Reconstructions (odd row: real, even row: reconstruction).}
         \label{fig:apndx_cifar_recons}
     \end{subfigure}
        \caption{Generated samples (left) and reconstructions (right) of test data from an unconditional S-IntroVAE trained on CIFAR-10.}
        \label{fig:fig:apndx_cifar}
\end{figure}

\begin{figure}
     \centering
        \begin{subfigure}[b]{0.4\textwidth}
             \centering
             \includegraphics[width=\textwidth]{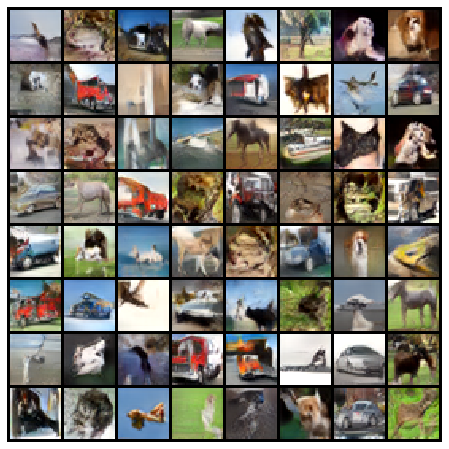}
             \caption{Generated samples (FID: 4.07).}
             \label{fig:apndx_cifar_samples_cond}
        \end{subfigure}
     \begin{subfigure}[b]{0.4\textwidth}
         \centering
         \includegraphics[width=\textwidth]{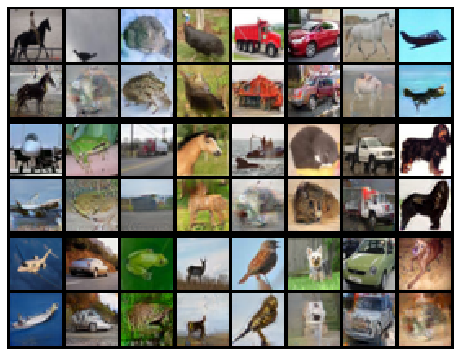}
         \caption{Reconstructions (odd row: real, even row: reconstruction).}
         \label{fig:apndx_cifar_recons_cond}
     \end{subfigure}
        \caption{Generated samples (left) and reconstructions (right) of test data from a class-conditional S-IntroVAE trained on CIFAR-10.}
        \label{fig:fig:apndx_cifar_cond}
\end{figure}

\paragraph{CelebA-HQ dataset} Results from a style-based S-IntroVAE trained on CelebA-HQ at resolution 256$\times$256 (FID: 18.63) are presented in Figure \ref{fig:image_samples_celeb}. Additional random (i.e., no cherry-picking) generated images from a style-based S-IntroVAE trained on CelebA-HQ at resolution 256$\times$256 are presented in Figure \ref{fig:apndx_celeb_samples} and random reconstructions of unseen data during training are presented in Figure \ref{fig:apndx_celeb_recons}.

\begin{center}
\begin{figure}
    \centering
    \includegraphics[width=0.8\textwidth]{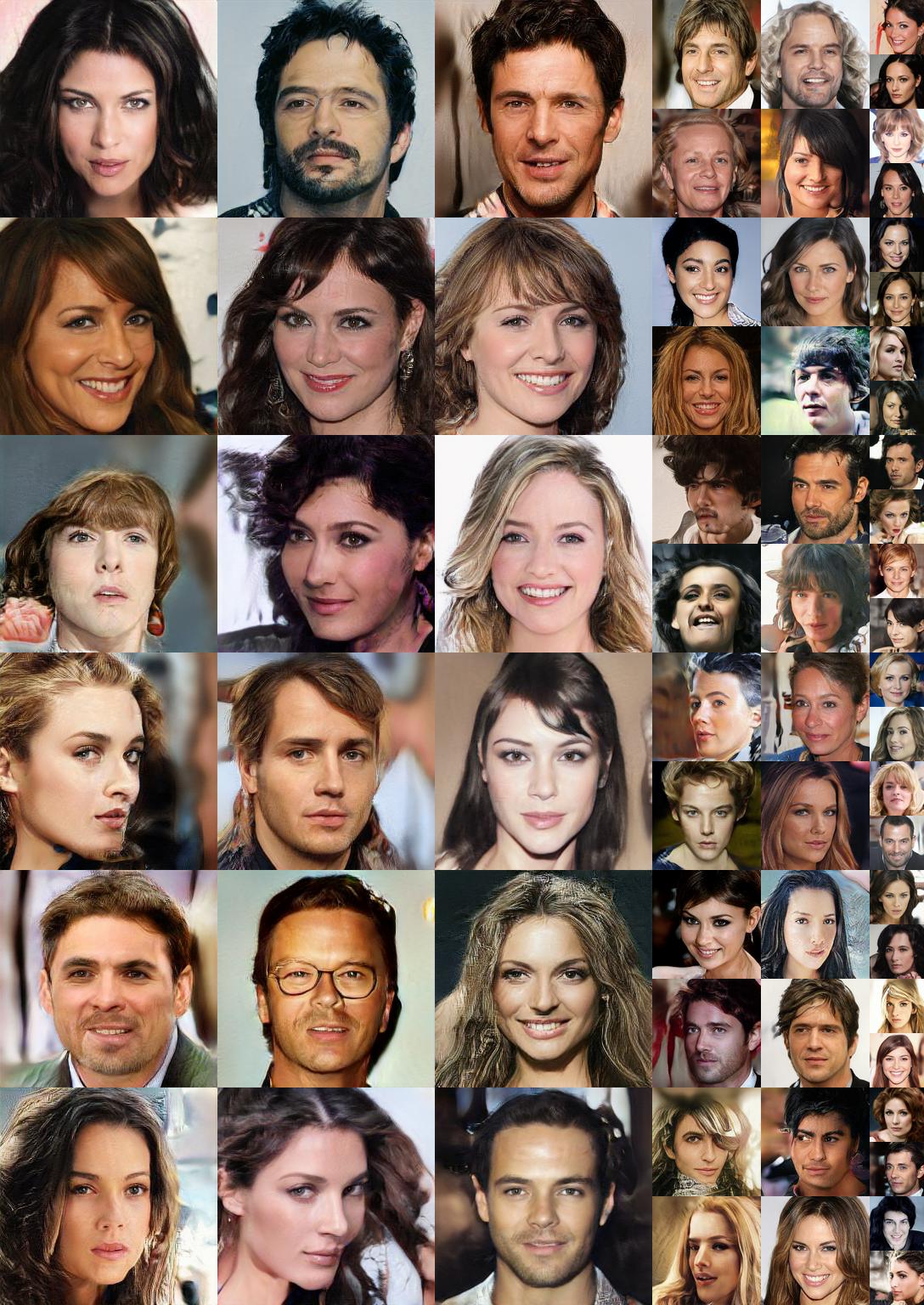}
    \caption{Generated samples from a style-based S-IntroVAE trained on CelebA-HQ at 256x256 resolution (FID: 18.63).}
    \label{fig:apndx_celeb_samples}
\end{figure}
\end{center}

\begin{center}
\begin{figure}
    \centering
    \includegraphics[width=0.8\textwidth]{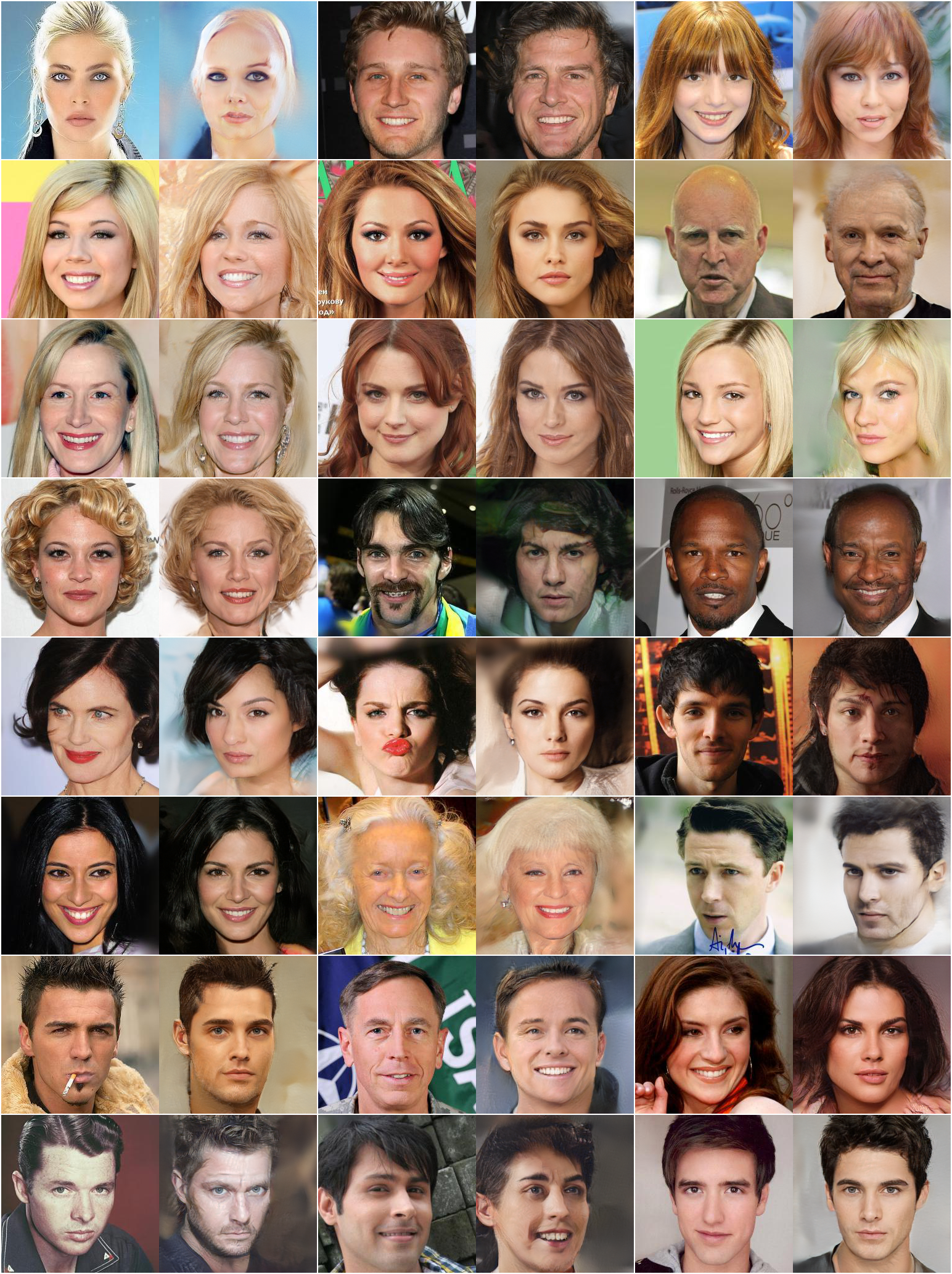}
    \caption{Reconstructions of test data from a style-based S-IntroVAE trained on CelebA-HQ at 256x256 resolution (left: real, right: reconstruction).}
    \label{fig:apndx_celeb_recons}
\end{figure}
\end{center}

\begin{figure}
     \centering
        \begin{subfigure}[b]{0.4\textwidth}
             \centering
             \includegraphics[width=\textwidth]{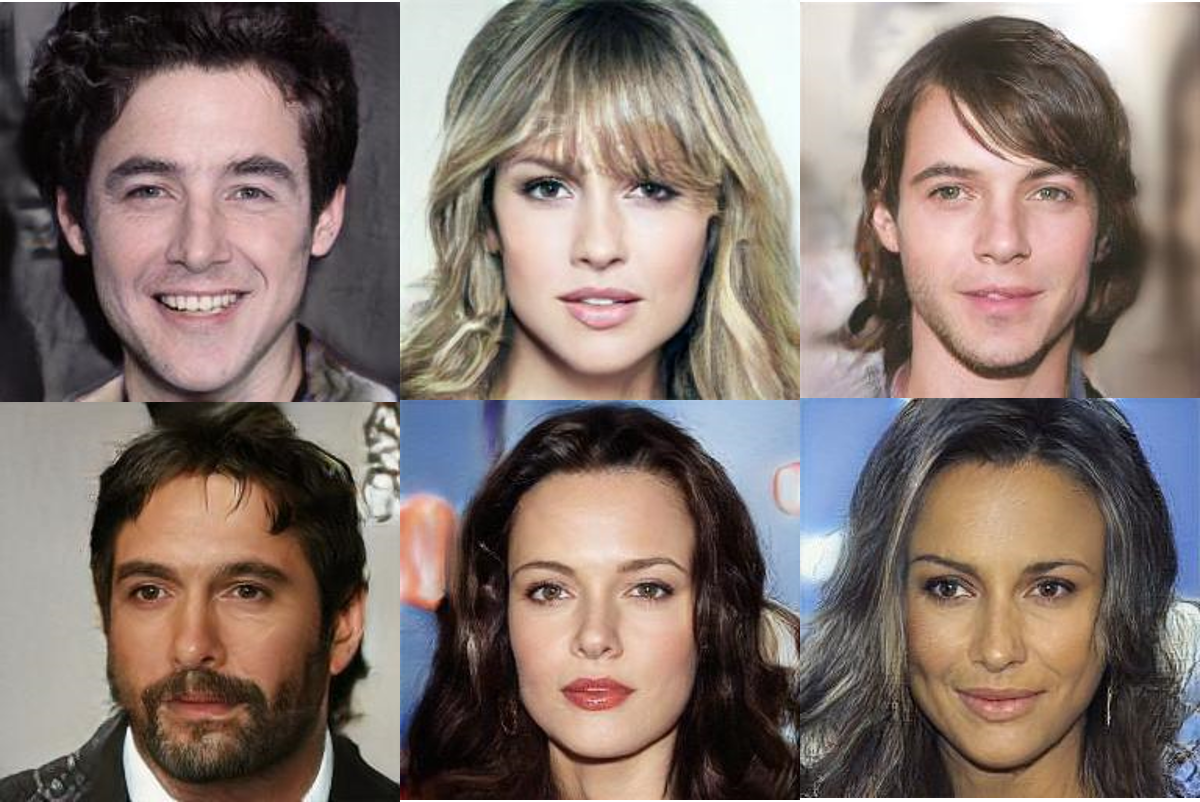}
             \caption{Generated samples from S-IntroVAE (FID: 18.63).}
             \label{fig:samples_celebhq}
        \end{subfigure}
     \begin{subfigure}[b]{0.4\textwidth}
         \centering
         \includegraphics[width=0.5\textwidth]{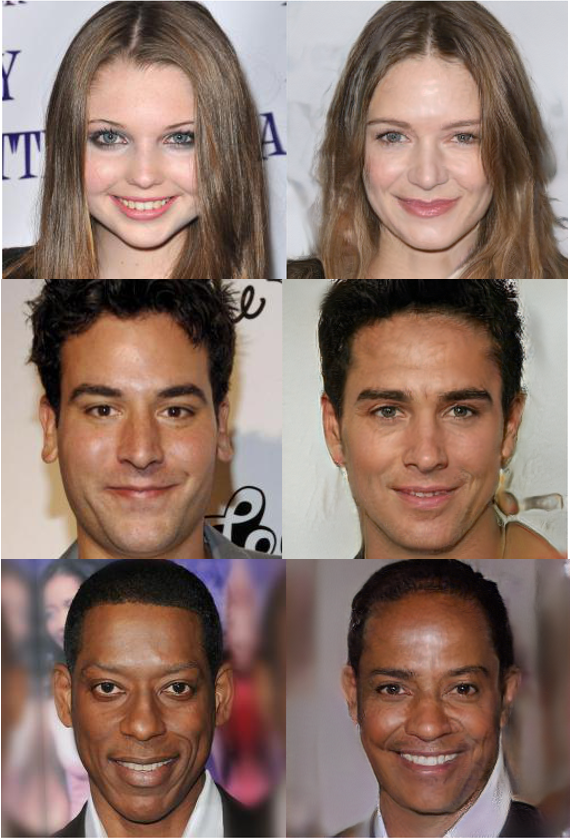}
         \caption{Reconstructions (left: real, right: reconstruction).}
         \label{fig:celeba_recons}
     \end{subfigure}
        \caption{Generated samples (left) and reconstructions (right) of test data from a style-based S-IntroVAE trained on CelebA-HQ at 256$\times$256 resolution. It is recommended to zoom-in.}
        \label{fig:image_samples_celeb}
\end{figure}

\paragraph{FFHQ dataset} Additional random (i.e., no cherry-picking) generated images from a style-based S-IntroVAE trained on FFHQ at resolution 256$\times$256 are presented in Figure \ref{fig:apndx_ffhq_samples} (FID: 17.55) and random reconstructions of unseen data during training are presented in Figure \ref{fig:apndx_ffhq_recons}.

\paragraph{LSUN Bedroom} LSUN Bedroom is a subset of the larger LSUN \cite{yu15lsun} dataset, and includes a training set of 3,033,042 images of different bedrooms. We train a style-based S-IntroVAE at a resolution of 128$\times$128. Samples from the trained model are presented in Figure \ref{fig:apndx_bedroom_samp}, and we report FID of 15.88.

\begin{center}
\begin{figure}
    \centering
    \includegraphics[width=0.8\textwidth]{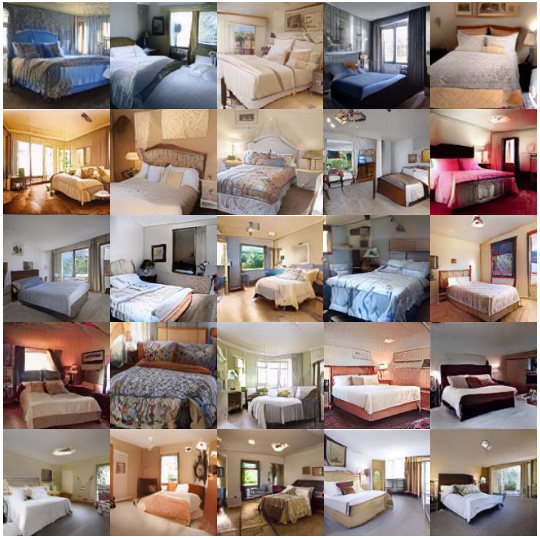}
    \caption{Samples a style-based S-IntroVAE trained on LSUN Bedroom at 128$\times$128 resolution (FID: 15.88).}
    \label{fig:apndx_bedroom_samp}
\end{figure}
\end{center}

\begin{center}
\begin{figure}
    \centering
    \includegraphics[width=0.8\textwidth]{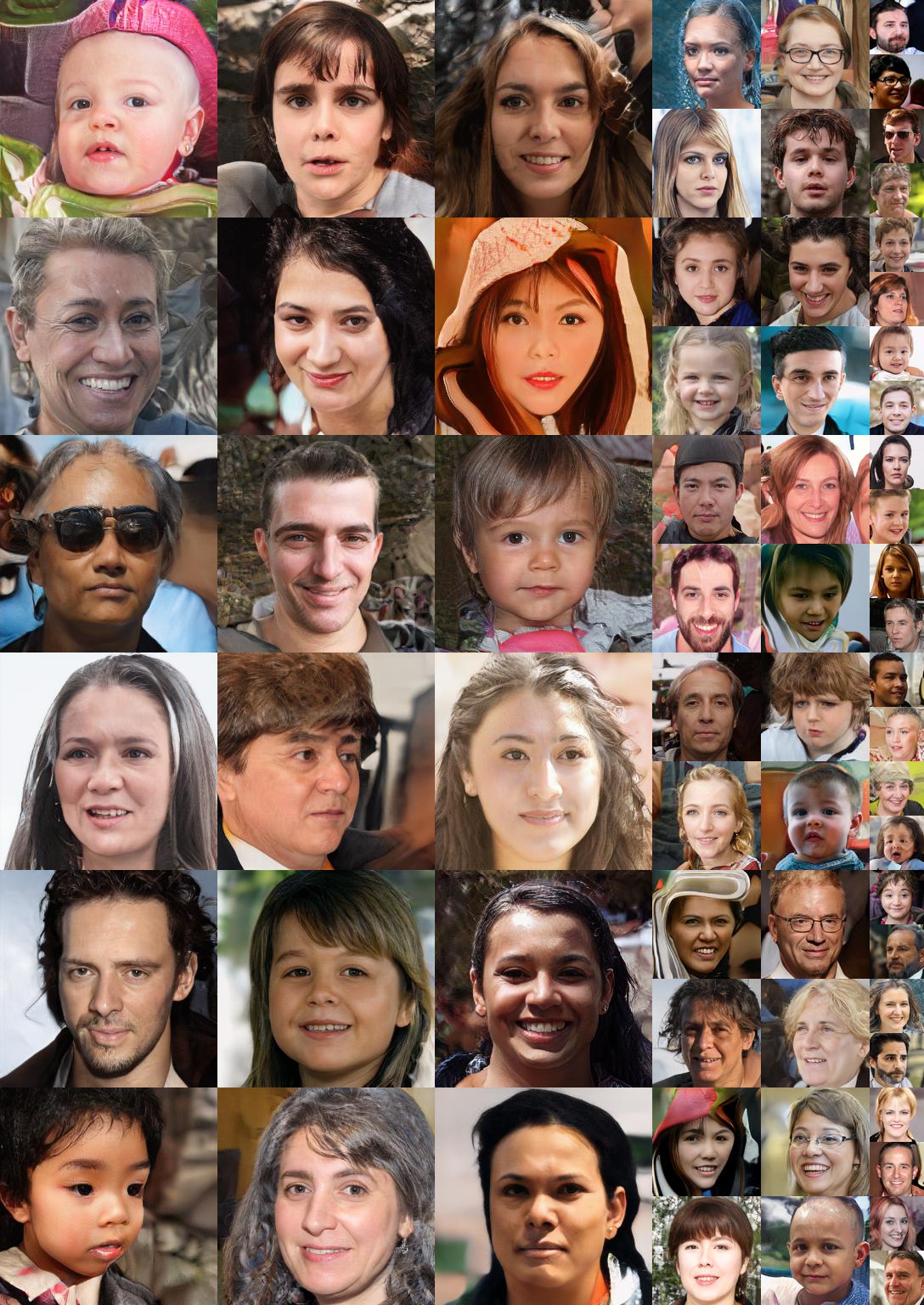}
    \caption{Generated samples from a style-based S-IntroVAE trained on FFHQ at 256x256 resolution (FID: 17.55).}
    \label{fig:apndx_ffhq_samples}
\end{figure}
\end{center}

\begin{center}
\begin{figure}
    \centering
    \includegraphics[width=0.8\textwidth]{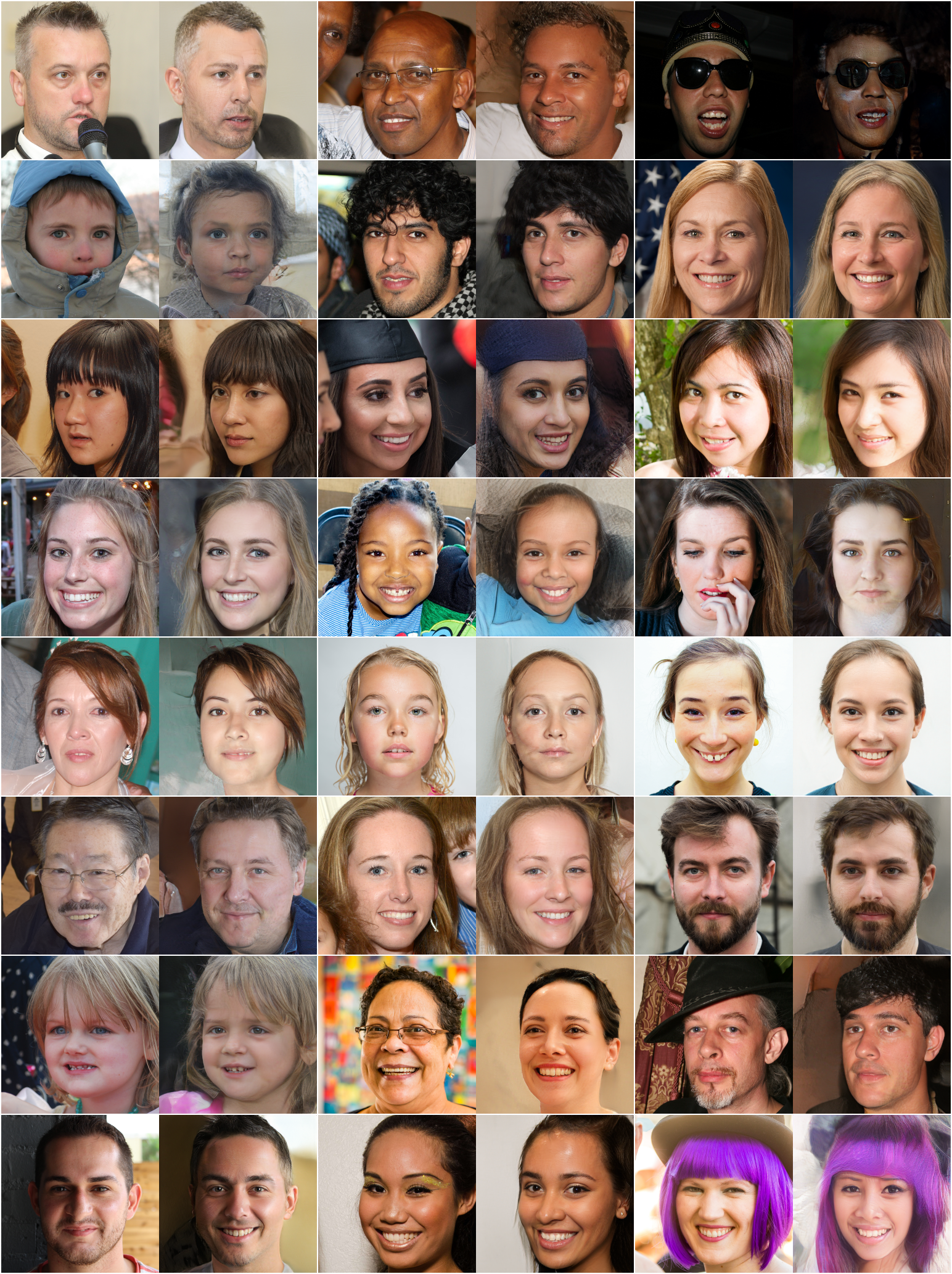}
    \caption{Reconstructions of test data from a style-based S-IntroVAE trained on FFHQ at 256x256 resolution (left: real, right: reconstruction).}
    \label{fig:apndx_ffhq_recons}
\end{figure}
\end{center}

\subsubsection{Interpolation in the Latent Space}
\label{apndx_interpolation}
One of the desirable properties of VAEs is the continuous learned latent space. 
Figure \ref{fig:apndx_celeb_interpolation} shows smooth interpolation between the latent vectors of four images from S-IntroVAE trained on the CelebA-HQ dataset. The interpolation is performed as follows: the four images are encoded to the latent space, and the latent codes serve as the corners of a square. We then perform 7-step linear interpolation between the latent codes of the corner images, such that each intermediate code is a mixture of the corner latent code, depending on the location on the grid. The intermediate latent codes are then decoded to produce the images comprising the square. Mathematically, let $z_a, z_b, z_c$ and $z_d$ denote the latent codes of images $X_a, X_b, X_c$ and $X_d$, respectively. The intermediate latent code $z_m$ is constructed as follows: $$ z_m = z_a\cdot (1 - \frac{i}{7})(1-\frac{j}{7}) + z_b \cdot \frac{j}{7}(1- \frac{i}{7}) +z_c \cdot (1- \frac{j}{7})\frac{i}{7} +z_d \cdot \frac{i}{7} \cdot \frac{j}{7}, $$
where $i, j = 1, .., 6$ denote the current location on the grid.

\begin{center}
\begin{figure}
    \centering
    \includegraphics[width=0.8\textwidth]{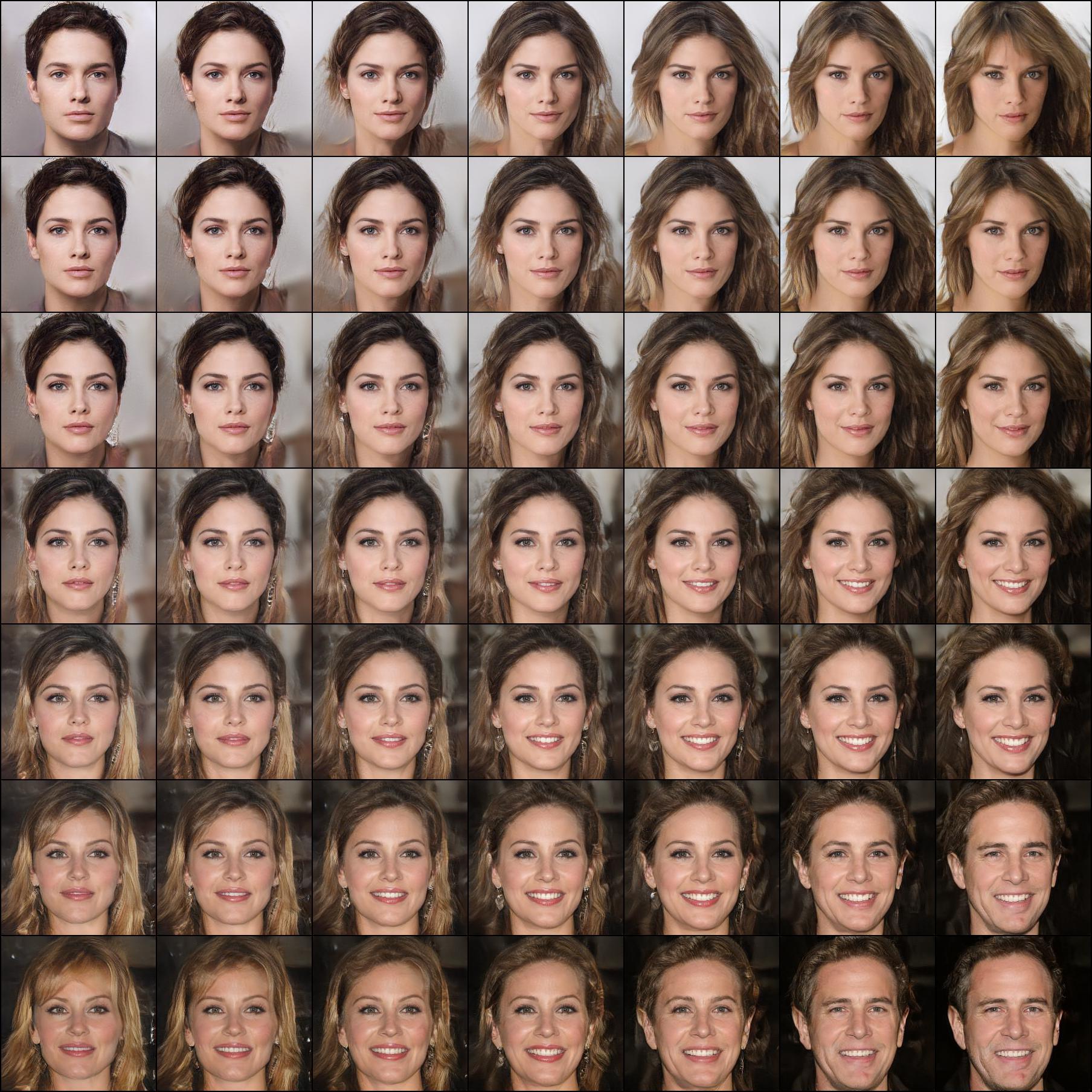}
    \caption{Interpolation in the latent space between four images, using a style-based S-IntroVAE trained on CelebA-HQ at 256x256 resolution.}
    \label{fig:apndx_celeb_interpolation}
\end{figure}
\end{center}

\subsubsection{Image Translation}
\label{apndx_translation}
We provide further image translation results for the Cars3D dataset in Figure \ref{fig:apndx_cars3d} and for the KTH dataset in Figure \ref{fig:apndx_kth}. The content transfer is performed as follows: for given two images, we encode both of them, and use the class latent code of the first one and the content latent code from the second one as input to the decoder. The output image should contain an object from the class of the first image (e.g., car model or person identity), with the content of the second (e.g. rotation or skeleton position).

\begin{center}
\begin{figure}
    \centering
    \includegraphics[width=0.5\textwidth]{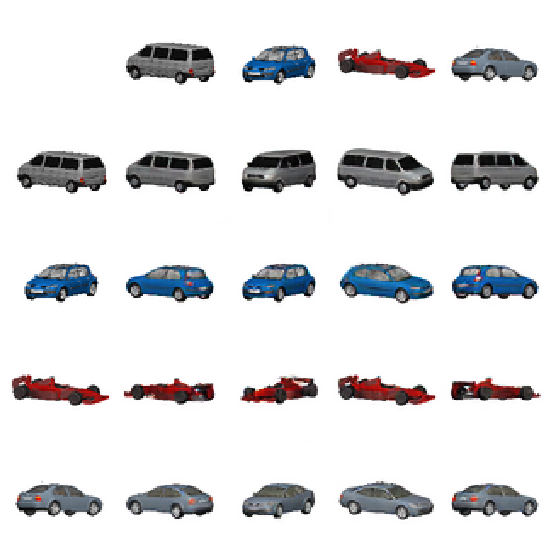}
    \caption{Qualitative results for content transfer on test data from the Cars3D dataset. The class is the car model, and the content is the rotation and azimuth.}
    \label{fig:apndx_cars3d}
\end{figure}
\end{center}

\begin{center}
\begin{figure}
    \centering
    \includegraphics[width=0.5\textwidth]{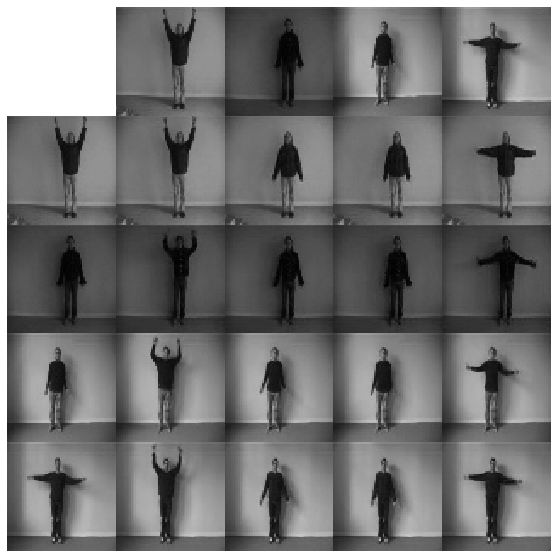}
    \caption{Qualitative results for content transfer on test data from the KTH dataset. The class is the person identity, and the content is the skeleton position.}
    \label{fig:apndx_kth}
\end{figure}
\end{center}

\subsection{Posterior Collapse}
\label{apndx:collapse}
\textit{Posterior collapse} \cite{Bowman_2016}, often occurs in image, text or autoregressive-based VAEs, happens when the approximate posterior distribution collapses onto the prior completely, that is, a trivial optimum is reached, a solution where the generator ignores the latent variable $z$ when generating $x$, and the KL term in the ELBO vanishes. Preventing posterior collapse has been addressed previously \cite{goyal2017zforcing, havrylov2020preventing, long2019preventing}, mainly by annealing the KL coefficient term ($\beta_{kl}$), adding auxiliary costs or changing the cost function altogether. \cite{gabbay2019demystifying} also noticed that when using a VAE formulation to train the specific disentanglement-oriented architecture on images, the KL term vanishes and the learned representations are uninformative. Empirically, posterior collapse can happen when the optimization of the VAE is more focused on the KL term, i.e., when $\beta_{kl} > \beta_{rec}$. Interestingly, we found that for the same $\beta_{kl}$ and $\beta_{kl}$, the expELBO term in the encoder's objective adds a 'repulsion' force that prevents this collapse. This is demonstrated on the 2D "8 Gaussians" dataset in Figure \ref{fig:2d_collapse}, where we train a standard VAE with $\beta_{kl}=1.0$ and $\beta_{rec}=0.5$, and a Soft-IntroVAE model with the same hyperparameters, but with $\beta_{neg}=5.0$. For the standard VAE, the KL term quickly vanishes during training, resulting in a trivial solution where the decoder ignores the latent variable $z$ when generating $x$. Moreover, \cite{gabbay2019demystifying} analysis showed that using a standard $\beta$-VAE for the image translation task results in sub-optimal results compared to a regular autoencoder due to the KL term vanishing. Our results on the image translation task show that with the added objectives of S-IntroVAE, it is possible to use a VAE for this task.

\begin{center}
\begin{figure}
    \centering
    \includegraphics[width=0.6\textwidth]{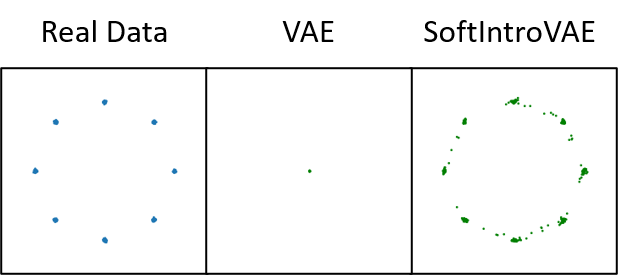}
    \caption{Demonstration of posterior collapse. Generated samples from trained models are shown, where both the standard VAE and S-IntroVAE were trained on the "8 Gaussians" 2D dataset with $\beta_{kl}=1.0$ and $\beta_{rec}=0.5$, and for S-IntroVAE, $\beta_{neg}=5.0$. For the standard VAE, the KL term vanishes, resulting in posterior collapse.}
    \label{fig:2d_collapse}
\end{figure}
\end{center}

\clearpage
\subsection{Theoretical Results}
\label{apndx:sec_theory}
In this section, we analyze the equilibrium of the S-IntroVAE model. We analyze the case of a general $\alpha\geq 1$, and the results in the main text for $\alpha=1$ are a special case. For the readers ease, we first recapitulate our definitions.
Recall that the encoder is represented by the approximate posterior distribution $q \doteq q(z|x)$ and that the decoder is represented using $d \doteq p_d(x|z)$. These are the controllable distributions in our generative model. The latent prior is denoted $p(z)$ and is not controlled. Slightly abusing notation, we also denote $p_d(x) = \mathbb{E}_{p(z)}[p_d(x|z)]$. The data distribution is denoted $p_{data}(x)$. For some distribution $p(x)$, let  $H(p) = -\mathbb{E}\left[ \log p(x) \right]$ denote its Shannon entropy.  

The ELBO, denoted $W(x;d,q)$, is given by:
\begin{equation}
    W(x;d,q) \doteq \mathbb{E}_{q(z|x)}\left[ \log p_d(x|z)\right] - KL(q(z|x) || p(z)).
\end{equation}
From the Radon-Nikodym Theorem~\cite{Capinski2004} of measure theory the following equality holds: 
\begin{equation}
\label{eq_rk_1}
\begin{split}
    \mathbb{E}_{z\sim p_z(z)}\left[\exp(\alpha W(D_{\theta}(z); d, q)) \right] =\mathbb{E}_{x\sim p_d(x)}\left[\exp(\alpha W(x; d, q)) \right],
\end{split}
\end{equation}
and similarly:
\begin{equation}
\label{eq_rk_2}
    \mathbb{E}_{z\sim p_z(z)}\left[W(D_{\theta}(z); d, q) \right] = \mathbb{E}_{x\sim p_d(x)}\left[W(x; d, q) \right].
\end{equation}
The ELBO satisfies the following property:
\begin{equation}\label{eq:elbo_lb}
    W(x;d,q) = \log p_d(x) - KL(q(z|x) || p_d(z|x)) \leq \log p_d(x).
\end{equation}
We consider a non-parametric setting, where $d$ and $q$ can be any distribution. For some $z$, let $D(z)$ denote a sample from $p_d(x|z)$. The objective functions for $q$ and $d$ are given by (note that we drop the dependence on $\theta,\phi$ because of the non-parametric setting):
\begin{equation}
\begin{split}
    \mathcal{L}_{E}(x,z) &=  W(x; d, q) - \frac{1}{\alpha} \exp(\alpha W(D(z)); d, q), \\
    \mathcal{L}_{D}(x,z) &= W(x; d, q) + \gamma W(D(z); d, q),
\end{split}
\end{equation}
where $\alpha\geq 1 $ and $\gamma \geq 0$ are hyper-parameters. The complete S-IntroVAE objective, takes an expectation of the losses above over real and generated samples:
\begin{equation}\label{eq:losses}
    \begin{split}
        L_q(q,d) &= \mathbb{E}_{p_{data}}\left[ W(x;q,d)\right] - \mathbb{E}_{p_{d}}\left[ \alpha^{-1}\exp(\alpha W(x;q,d))\right], \\
        L_d(q,d) &= \mathbb{E}_{p_{data}}\left[ W(x;q,d)\right] + \gamma \mathbb{E}_{p_{d}}\left[ W(x;q,d)\right].
    \end{split}
\end{equation}
A Nash equilibrium point $(q^*,d^*)$ satisfies $L_q(q^*,d^*) \geq L_q(q,d^*)$ and 
$L_d(q^*,d^*) \geq L_d(q^*,d)$
for all $q,d$. Given some $d$, let $q^*(d)$ satisfy $L_q(q^*(d),d) \geq L_q(q,d)$ for all $q$. 

\begin{lemma}\label{lem:max_L_q}
If $p_{d}(x) \leq p_{data}(x)^{\frac{1}{\alpha + 1}}$ for all $x$ for which $p_{data}(x)>0$, we have that $q^*(d)$ satisfies $q^*(d)(z|x)=p_d(z|x)$, and $W(x;q^*(d),d) = \log p_d(x)$.
\end{lemma}
\begin{proof}
Plugging \eqref{eq:elbo_lb} in \eqref{eq:losses} we have that: 
\begin{equation}
\begin{split}
     L_q(q,d) &= \mathbb{E}_{p_{data}}\left[ \log p_d(x) - KL(q(z|x) || p_d(z|x))\right] -\frac{1}{\alpha} \mathbb{E}_{p_{d}}\left[ \exp(\alpha\log p_d(x) - \alpha KL(q(z|x) || p_d(z|x)))\right] \\
     &= \mathbb{E}_{p_{data}}\left[ \log p_d(x) -  KL(q(z|x) || p_d(z|x))\right] - \frac{1}{\alpha}\mathbb{E}_{p_{d}}\left[ p_d^{\alpha}(x)\exp(- \alpha KL(q(z|x) || p_d(z|x)))\right] \\
     &= \sum_x p_{data}(x)(\log p_d(x) - KL(q(z|x) || p_d(z|x))) - \frac{1}{\alpha} p_{d}^{\alpha+1}(x)\exp(- \alpha KL(q(z|x) || p_d(z|x))).
\end{split}
\end{equation}
Consider some $x$ for which $p_{data}(x)>0$. We have that $q^*(d)(z|x)$ is the maximizer of
\begin{equation}
\begin{split}
     & p_{data}(x)\left(\log p_d(x) -  KL(q(z|x) || p_d(z|x)) -\frac{1}{\alpha} \cdot \frac{p_{d}^{{\alpha+1}}(x)}{p_{data}(x)}\exp(- \alpha  KL(q(z|x) || p_d(z|x)))\right). \\
\end{split}
\end{equation}
Consider now the function $g(y) = y - \frac{a}{\alpha} \exp(\alpha y)$. We have that $g'(y)=1-a \exp(\alpha y)$, and therefore the function obtains a maximum at $y=-\frac{1}{\alpha}\log (a)$. In our case, $a=\frac{p_{d}^{\alpha+1}(x)}{p_{data}(x)}$ and $y=- KL(q(z|x) \| p_d(z|x)) \leq 0$. Therefore, if $\frac{ p_{d}^{\alpha+1}(x)}{p_{data}(x)} > 1$, then the maximum is obtained for $- KL(q(z|x) || p_d(z|x)) = -\frac{1}{\alpha}\log\left( \frac{ p_{d}^{\alpha+1}(x)}{p_{data}(x)} \right)$, and if $\frac{ p_{d}^{\alpha+1}(x)}{p_{data}(x)} \leq 1$ then the maximum is obtained for $- KL(q(z|x) || p_d(z|x))=0$.

For $x$ such that $p_{data}(x)=0$, we have that $q(z|x)$ is the maximizer of $-\frac{1}{\alpha}\cdot p_{d}^{\alpha+1}(x)\exp(- \alpha KL(q(z|x) || p_d(z|x)))$. Since $KL(\cdot || \cdot) \geq 0$ and $p_{d}^{\alpha+1}(x)\geq 0$, a maximum is obtained for $- KL(q(z|x) || p_d(z|x))=0$.
Thus, given the assumption in the Lemma, for every $x$ the maximum is obtained for $KL(q(z|x) || p_d(z|x))=0$.
\end{proof}

Define $d^*$ as follows:
\begin{equation}\label{eq:pdstar_apndx}
    d^* \in \argmin_{d} \left\{KL(p_{data} || p_{d}) + \gamma H(p_{d}(x)) \right\},
\end{equation}
where $H(\cdot)$ is the Shannon entropy.
We make the following assumption.
\begin{assumption}\label{ass:ass_1_alpha}
For all $x$ such that $p_{data}( x)>0$ we have that $p_{d^*}(x) \leq p_{data}(x)^{\frac{1}{\alpha + 1}}$.
\end{assumption}
For $\alpha=1$, we get that Assumption \ref{ass:ass_1_alpha} is equivalent to Assumption \ref{ass:ass_1} in the main text.
We now claim that the equilibrium point of the optimization in \eqref{eq:losses} is $(q^*(d^*),d^*)$ as defined in \eqref{eq:pdstar_apndx}.

\begin{theorem}
Denote $q^* = p_{d^*}(z|x)$, with $d^*$ defined in \eqref{eq:pdstar_apndx}, and let Assumption \ref{ass:ass_1_alpha} hold.
Then $\left(q^*,  {d^*}\right)$ is a Nash equilibrium of \eqref{eq:losses}.
\end{theorem}

\begin{proof}
From Lemma \ref{lem:max_L_q} we have that $q^*(d^*)(z|x) = p_{d^*}(z|x)$.

Let $d$ be some decoder parameters (i.e., $p_{d}(x|z)$).
From \eqref{eq:elbo_lb} we have that $W(x;q^*(d),d) = \log (p_{d}(x)) - KL(q^*(z|x) \| p_d(z|x))$.
Now, we have that 
\begin{equation}\label{eq:dec_loss_2}
\begin{split}
    L_d(q^*(d),d) &= \mathbb{E}_{p_{data}}\left[ W(x;q^*(d),d)\right] + \gamma \mathbb{E}_{p_{d}}\left[ W(x;q^*(d),d)\right] \\
    &= \mathbb{E}_{p_{data}}\left[ \log (p_{d}(x)) - KL(q^*(z|x) \| p_d(z|x)) \right] + \gamma \mathbb{E}_{p_{d}}\left[ \log (p_{d}(x)) - KL(q^*(z|x) \| p_d(z|x))\right] \\
    &= -KL(p_{data} \| p_{d}) + \mathbb{E}_{p_{data}}\left[ \log (p_{data}(x)) \right] -\gamma H(p_{d}(x))\\
    &- \mathbb{E}_{p_{data}}\left[KL(q^*(z|x) \| p_d(z|x)) \right] -\gamma \mathbb{E}_{p_{d}}\left[ KL(q^*(z|x) \| p_d(z|x))\right].
\end{split}
\end{equation}
Since $KL(q^*(d) \| p_d(z|x)) \! \geq \! 0 \! = \! KL(q^*(d^*) \| p_{d^*}(z|x))$, and $p_{d^*} = \argmin_{p_{d}} \left\{KL(p_{data} \| p_{d}) + \gamma H(p_{d}(x)) \right\}$, we have that $d^* \in \argmax_d L_d(q^*(d),d)$. Also, since $KL(q^* \| p_d(z|x))  \geq  0  =  KL(q^* \| p_{d^*}(z|x))$, we have that $d^* \in \argmax_d L_d(q^*,d)$.
We conclude that $(q^*,d^*)$ is a Nash equilibrium of \eqref{eq:losses}.
\end{proof}

Theorem \ref{thm:equilibrium} assumes that  $p_{d^*}(x) \leq p_{data}(x)^{\frac{1}{\alpha + 1}}$ for all $x$. We now claim that for any $p_{data}$, there exists some $\gamma>0$ such that this assumption holds.

\begin{theorem}
For any $p_{data}$, there exists $\gamma>0$ such that $p_{d^*}$, as defined defined in \eqref{eq:pdstar_apndx}, satisfies Assumption \ref{thm:equilibrium}.
\end{theorem}

\begin{proof}
We will show that for $\gamma=0$ the condition holds, and that $p_{d^*}$ is continuous in $\gamma$. 

Since $\alpha \geq 1$, for $\gamma=0$ we have that $p_{d^*}=p_{data}$. Therefore $\frac{( p_{d^*}(x))^{\alpha+1}}{p_{data}(x)} =  p_{d^*}^{\alpha}(x) \leq 1$.\footnote{The condition $p_d(x)\leq 1$ is obvious for discrete distributions. For continuous distributions, it is satisfied in a differential sense $p_d(x)dx\leq 1$, since $\int_x p_d(x)dx = 1$ and $p_d(x)\geq 0$.}

By Theorem 2 of Milgrom and Segal~\cite{milgrom2002envelope} (the Envelope theorem) we have that the value function $V(\gamma) = \min_{d} \left\{KL(p_{data} \| p_{d}) + \gamma H(p_{d}(x)) \right\}$ is continuous in $\gamma$.  Therefore, for every $\epsilon>0$ there exists some $\gamma$ for which $V(\gamma) - V(0) \leq \epsilon$, which yields
\begin{equation}\label{eq:proof_1}
\begin{split}
     &\min_{d} \left\{KL(p_{data} \| p_{d}) + \gamma H(p_{d}(x)) \right\} - \min_{d} \left\{KL(p_{data} \| p_{d})\right\} \\
     &= \min_{d} \left\{KL(p_{data} \| p_{d}) + \gamma H(p_{d}(x)) \right\}\leq \epsilon.
\end{split}
\end{equation}
Let $d^*$ satisfy $d^* \in \argmin_{d} \left\{KL(p_{data} \| p_{d}) + \gamma H(p_{d}(x)) \right\}$. Since the entropy $H$ is non-negative, we have from \eqref{eq:proof_1} that 
$$
KL(p_{data} \| p_{d^*}) \leq \epsilon.
$$
Let $D(p_{data} \| p_{d^*}) = \sup_x |p_{data}(x) - p_{d^*}(x) |$ denote the total variation distance. From Pinsker's inequality we have that
\begin{equation}\label{eq:proof_2}
    D(p_{data} \| p_{d^*}) \leq \sqrt{0.5 KL(p_{data} \| p_{d^*})} \leq \sqrt{0.5 \epsilon}.
\end{equation}
Choose $\epsilon$ such that 
\begin{equation}\label{eq:proof_3}
    \sqrt{0.5 \epsilon} \leq \min_{x: p_{data}(x)>0} \left\{ -p_{data}(x) + p_{data}(x)^{\frac{1}{\alpha + 1}}\right\},
\end{equation}
and note that since $p_{data}(x)\leq 1$ and $\alpha \geq 1$, then $-p_{data}(x)+p_{data}(x)^{\frac{1}{\alpha + 1}}\geq 0$, and therefore we can find an $\epsilon>0$ that satisfies \eqref{eq:proof_3}. We thus have that for any $x$ such that $p_{data}(x)>0$:
\begin{equation}
    p_{d^*}(x) \leq p_{data}(x) + \sqrt{0.5 \epsilon} \leq p_{data}(x)^{\frac{1}{\alpha + 1}},
\end{equation}
where the first inequality is from the definition of the total variation distance and \eqref{eq:proof_2}, and the second inequality is by \eqref{eq:proof_3}.
\end{proof}

\subsection{Out-of-Distribution Detection Experiment}
\label{sec:apndx_ood}
One common application of likelihood-based generative models is detecting novel data, or out-of-distribution (OOD) detection~\cite{nalisnick2018deep, choi2018waic, zisselman2020deep}. Typically in an unsupervised setting, where only in-distribution data is seen during training, the inference modules in these models are \textit{expected} to assign in-distribution data high likelihood, while OOD data should have low likelihood. Surprisingly, Nalisnick et al.~\cite{nalisnick2018deep} showed that for some image datasets, density-based models, such as VAEs and flow-based models, cannot distinguish between images from different datasets, when trained only on one of the datasets. Evidently, \cite{nalisnick2018deep} showed this phenomenon occurs when pairing several popular datasets such as FashionMNIST vs MNIST and SVHN vs CIFAR-10. Modelling the score of each data point by the ELBO, we perform a similar experiment and measure the OOD detection performance by the area under the receiver operating characteristic curve (AUROC) of both the standard VAE and Soft-IntroVAE. Our results, reported in Table \ref{tab:ood}, show that for the CIFAR10-SVHN pair, the standard VAE performs poorly, confirming the results of \cite{nalisnick2018deep}, while Soft-IntroVAE outperforms it by a large margin. 
Motivated by our findings, we posit that by exploring better generative models, the likelihood-based approach for OOD detection may provide promising results.

\begin{center}
    \begin{table}[]
    \centering
    \begin{tabular}{|l|l|l|l|l|l|}
    \hline
         \textbf{In-distribution (train)} & \textbf{ Out-of-distribution (test)}  & VAE & Soft-IntroVAE \\ \hline
    MNIST & FashionMNIST & 0.992 $\pm$ 0.002 & \textbf{0.999 $\pm$ 0.0002}  \\ \hline
    FashionMNIST  & MNIST & 0.996 $\pm$ 0.0009 & \textbf{0.999 $\pm$ 0.0004} \\ \hline
    CIFAR10  & SVHN & 0.378 $\pm$ 0.01 & \textbf{0.9987 $\pm$ 0.008} \\ \hline 
    SVHN & CIFAR10 & 0.936 $\pm$ 0.003 & \textbf{0.966 $\pm$ 0.02 }  \\ \hline
    \end{tabular}
    \caption{Comparison of AUROC scores for OOD, where the in-distribution (train) and out-of-distribution (test) are different datasets and the ELBO is used for the score threshold. Results are averaged over 3 seeds. }
    \label{tab:ood}
\end{table}
\end{center}

\end{document}